\def\isReadyToSubmit{0}        \def\includeAuthor{1}     \def\sandy{0}             

\ifnum\sandy=0
  \documentclass[sigplan,10pt]{acmart}
\else
  \documentclass[letterpaper,single,10pt]{article}
\fi

\usepackage{times}
\usepackage{graphicx}
\usepackage{subcaption}
\usepackage{bbold}
\usepackage{array} \usepackage{amsmath,amsfonts,amssymb,amsthm}
\usepackage{color}
\usepackage{xspace} \usepackage{algorithm} \usepackage[noend]{algpseudocode} \usepackage{wrapfig} \usepackage{flushend} \usepackage{multirow} \usepackage{booktabs} \usepackage{anyfontsize} \usepackage{paralist} \usepackage[font=small]{caption}
\usepackage{courier} \usepackage{verbatim}
\usepackage{calc}
\usepackage{pbox}
\usepackage{tcolorbox}
\ifnum\sandy=1
  \usepackage{setspace}
\fi
\usepackage{listings}
\usepackage{color,soul}
\usepackage{enumitem}
\usepackage{cancel}
\usepackage{hyperref}
\usepackage{mdframed}
\usepackage[normalem]{ulem}

\usepackage[compact,small]{titlesec}
\titlespacing{\section}{6pt}{*.3}{*.1}
\titlespacing{\subsection}{4pt}{*.4}{*.2}
\titlespacing{\subsubsection}{2pt}{*.4}{*.2}

\usepackage{etoolbox}
\makeatletter
\patchcmd{\ttlh@hang}{\parindent\z@}{\parindent\z@\leavevmode}{}{}
\patchcmd{\ttlh@hang}{\noindent}{}{}{}
\makeatother

\theoremstyle{definition}

\newenvironment{ftheorem}{\begin{theorem}}{\end{theorem}}
\newenvironment{fproposition}{\begin{proposition}}{\end{proposition}}
\newenvironment{fproof}{\begin{proof}}{\end{proof}}
\newenvironment{fdefinition}{\begin{definition}}{\end{definition}}


\newcommand{\eg}{{\em e.g.,~}}

\def\F{Fig.~}
\def\T{Tab.~}
\def\L{List.~}

\newcommand{\ar}[3]{} \include{latexdefs}
\include{mathnotation}

 \newcommand{\heading}[1]{{\vspace{2pt}\noindent\bf{#1}}}

{\makeatletter
 \gdef\xxxmark{\expandafter\ifx\csname @mpargs\endcsname\relax \expandafter\ifx\csname @captype\endcsname\relax \marginpar{\textcolor{red}{xxx~}}\else
       \textcolor{red}{xxx~}\fi
   \else
     \textcolor{red}{xxx~}\fi}
 \gdef\xxx{\@ifnextchar[\xxx@lab\xxx@nolab}
 \long\gdef\xxx@lab[#1]#2{{\bf [\xxxmark \textcolor{red}{#2} ---{\sc #1}]}}
 \long\gdef\xxx@nolab#1{{\bf [\xxxmark \textcolor{red}{#1}]}}
\ifnum\isReadyToSubmit=1
   \long\gdef\xxx@lab[#1]#2{}\long\gdef\xxx@nolab#1{}
 \fi
}

{\makeatletter
 \gdef\edit{\@ifnextchar[\edit@lab\edit@nolab}
 \long\gdef\edit@lab[#1]#2{[\textcolor{red}{#2} ---{\sc #1}]}
 \long\gdef\edit@nolab#1{[\textcolor{red}{#1}]}
\ifnum\isReadyToSubmit=1
   \long\gdef\edit@lab[#1]#2{[#2]}
 \fi
}

\newcommand{\ignore}[1]{}

\definecolor{codegreen}{rgb}{0,0.6,0}
\lstdefinestyle{codestyle}{
    commentstyle=\color{codegreen},
    keywordstyle=\bfseries,
    basicstyle=\scriptsize\ttfamily,
    breakatwhitespace=false,         
    captionpos=b,                    
    keepspaces=true,                 
    numbers=left,                    
    numbersep=4pt,                  
    showspaces=false,                
    showstringspaces=false,
    showtabs=false,                  
    tabsize=2,
    abovecaptionskip=2pt,
    belowcaptionskip=-20pt,
xleftmargin=2em,
    rulesepcolor=\color{white},
    rulecolor=\color{white},
frame=single}
\newcommand{\hlc}[2][yellow]{{\sethlcolor{#1}\hl{#2}}}

\DeclareMathAlphabet\mathbfcal{OMS}{cmsy}{b}{n}

\newtheorem*{proposition*}{Proposition}
\newtheorem*{theorem*}{Theorem}
\newtheorem*{lemma*}{Lemma}
\newtheorem*{corollary*}{Corollary}
\theoremstyle{definition}

\def\e{\epsilon}  \def\d{\delta}  \def\eg{\epsilon_g}  \def\dg{\delta_g}  \def\egdg{(\eg,\dg)}  \def\ed{(\e,\d)}  \def\D{\mathcal{D}}
\def\M{\mathcal{Q}}

\def\Y{\mathcal{V}}
\def\O{V}
\def\ES{\mathcal{S}}
\def\A{\mathcal{A}}
\def\W{\mathcal{W}}
\def\World{\mathcal{W}}
\def\cL{\mathcal{L}}

\newcommand\Loss{\operatorname{Loss}}
\newcommand\block{\operatorname{blocks}}

\newcommand\cF{\mathcal{F}}

\newcommand\Exp{\mathbb{E}}
\newcommand\fopt{{f^{\star}}}
\newcommand\fdp{{f^{\operatorname{dp}}}}

\newcommand\ntrain{{n_{\operatorname{tr}}}}
\newcommand\ntest{{n_{\operatorname{te}}}}
\newcommand\ntestDPLB{\underline{n}^{\operatorname{dp}}_{\operatorname{te}}}
\newcommand\ntestDP{n^{\operatorname{dp}}_{\operatorname{te}}}
\newcommand\ntrainDP{n^{\operatorname{dp}}_{\operatorname{tr}}}
\newcommand\ntrainDPUB{\overline{n}^{\operatorname{dp}}_{\operatorname{tr}}}
\newcommand\ntrainDPLB{\underline{n}^{\operatorname{dp}}_{\operatorname{tr}}}
\newcommand\LUB{\overline{\cL}}
\newcommand\LLB{\underline{\cL}}

\newcommand\Lap{\operatorname{Laplace}}

\newcommand\BinUB{\overline{\operatorname{Bin}}}
\newcommand\BinLB{\underline{\operatorname{Bin}}}

\newcommand\ACCEPT{\textsc{Accept}\xspace}
\newcommand\REJECT{\textsc{Reject}\xspace}
\newcommand\RETRY{\textsc{Retry}\xspace}

\def\sysname{Sage\xspace}
\def\iterativetraining{privacy-adaptive training\xspace}
\def\IterativeTraining{Privacy-Adaptive Training\xspace}
\def\Iterativetraining{Privacy-adaptive training\xspace}

\def\blockcomposition{block composition\xspace}

\def\Blockcomposition{Block composition\xspace}

\begin{document}

\acmYear{2019}\copyrightyear{2019}
\setcopyright{acmcopyright}
\acmConference[SOSP '19]{SOSP '19: Symposium on Operating Systems Principles}{October 27--30, 2019}{Huntsville, ON, Canada}
\acmBooktitle{SOSP '19: Symposium on Operating Systems Principles, October 27--30, 2019, Huntsville, ON, Canada}
\acmPrice{15.00}
\acmDOI{10.1145/3341301.3359639}
\acmISBN{978-1-4503-6873-5/19/10}

\ifnum\sandy=1
  \doublespacing
\fi
\date{}  

\title{\huge
Privacy Accounting and Quality Control in the\\\sysname Differentially Private ML Platform
}

\ifnum\includeAuthor=1
\author{ {\rm Mathias L\'ecuyer}, {\rm Riley Spahn}, {\rm Kiran Vodrahalli}, {\rm Roxana Geambasu}, and {\rm Daniel Hsu}  \\
         Columbia University
       }
\else
  \author{\vspace{-0.5cm}{\rm Submission \#194}\vspace{-0.5cm}}
\fi

\pagestyle{plain}

\captionsetup{font=footnotesize}

\begin{abstract}

Companies increasingly expose machine learning (ML) models trained over sensitive user data to untrusted domains, such as end-user devices and wide-access model stores.  This creates a need to control the data's leakage through these models.
We present {\em \sysname}, a differentially private (DP) ML platform that bounds the cumulative leakage of training data through models.
\sysname builds upon the rich literature on DP ML algorithms and contributes pragmatic solutions to two of the most pressing systems challenges of global DP:
{\em running out of privacy budget} and {\em the privacy-utility tradeoff}.
To address the former, we develop {\em \blockcomposition}, a new privacy loss accounting method that leverages the growing database regime of ML workloads to keep training models endlessly on a sensitive data stream while enforcing a global DP guarantee for the stream.
To address the latter, we develop {\em \iterativetraining}, a process that trains a model on growing amounts of data and/or with increasing privacy parameters until, with high probability, the model meets developer-configured quality criteria.
\sysname's methods are designed to integrate with TensorFlow-Extended, Google's open-source ML platform.
They illustrate how a systems focus on characteristics of ML workloads enables pragmatic solutions that are not apparent when one focuses on individual algorithms, as most DP ML literature does.

\end{abstract}

\maketitle
\section{Introduction}
\label{sec:introduction}

Machine learning (ML) introduces a dangerous double standard in how companies protect user data.
In traditional applications, sensitive data is siloed and its accesses are carefully controlled.
Consider a messaging application: you do not expect everyone to have access to your messages.
Indeed, there is access control logic programmed within that application that mediates all accesses to the messages and determines who should see which messages.
But in an ML-driven application, the messages are not just used for the main functionality but also to train various models, such as an auto-complete model, a chat bot, or a recommendation model.
These models can leak what you wrote~\cite{carlini2018theSecretSharer} and despite that, they issue predictions to everyone in the world, are often shipped to end-user devices for faster predictions~\cite{baylor2017tfx, hazelwood2018applied, Li:michelangelo,ravi2017onDeviceMachineIntelligence}, and sometimes are shared across teams within the company more liberally than the data~\cite{hazelwood2018applied, Li:michelangelo, twitter-embeddings}.

There is perhaps a sense that because ML models aggregate data from multiple users, they obfuscate individuals' data and therefore warrant weaker protection than the data itself.
However, increasing evidence, both theoretical and empirical, suggests that ML models can indeed leak specifics about individual entries in their training sets.
Language models trained over users' emails for auto-complete have been shown to encode not only commonly used phrases but also social security numbers and credit card numbers that users may include in their private communications~\cite{carlini2018theSecretSharer}.
Collaborative filtering models, as used in recommenders, have been shown to leak specific information across users~\cite{calandrino2011you}.
Membership in a training set was shown to be inferable even when the attacker only has access to a model's external predictions~\cite{shokri2017membership}.
Finally, it has long been established theoretically that access to too many accurate linear statistics from a dataset -- as an adversary might have by observing periodic releases of models, which often incorporate statistics used for featurization -- {\em fundamentally allows reconstruction of the dataset}~\cite{dinurNissim2003revealing}.

As companies continue to disseminate many versions of models into untrusted domains, controlling the risk of data exposure becomes critical.
We present {\em \sysname}, an ML platform based on Google's TensorFlow-Extended (TFX) that uses differential privacy (DP) to bound the cumulative exposure of individual entries in a company's sensitive data streams through all the models released from those streams.
DP randomizes a computation over a dataset (e.g. training one model) to bound the leakage of individual entries in the dataset through the output of the computation (the model).
Each new DP computation increases the bound over data leakage, and can be seen as consuming part of a {\em privacy budget} that should not be exceeded; \sysname makes the process that generates all models and statistics preserve a {\em global DP guarantee}.

Sage builds upon the rich literature on DP ML {\em algorithms} (e.g.,~\cite{abadi2016deep,mcmahan2018aGeneral,mcsherry2009differentially}) and contributes pragmatic solutions to two of the most pressing {\em systems challenges} of global DP: (1) running out of privacy budget and (2) the privacy-utility tradeoff.
\sysname expects to be given training pipelines explicitly programmed to satisfy a parameterized DP guarantee.
It acts as a new access control layer in TFX that mediates data accesses by these pipelines and accounts for the cumulative privacy loss from them to enforce the global DP guarantee against the stream.
At the same time, \sysname provides the developers with: control over the quality of the models produced by the DP training pipelines (addresses challenge (2)); and the ability to release models endlessly without running out of privacy budget for the stream (addresses challenge (1)).

The key to addressing both challenges is the realization that ML workloads operate on {\em growing databases}, a model of interaction that has been explored very little in DP, and only with a purely theoretical and far from practical approach~\cite{cummings2018differential}.
Most DP literature, largely focused on {\em individual algorithms}, assumes either static databases (which do not incorporate new data) or online streaming (where computations do not revisit old data).
In contrast, {\em ML workloads} -- which consist of many algorithms invoked periodically -- operate on {\em growing databases}.
Across invocations of different training algorithms, the workload both incorporates new data and reuses old data, often times adaptively.
It is in that {\em adaptive reuse of old data coupled with new data} that our design of \sysname finds the opportunity to address the preceding two challenges in ways that are practical and integrate well with TFX-like platforms.

To address the running out of privacy budget challenge, we develop {\em \blockcomposition}, the first privacy accounting method that both allows efficient training on growing databases and avoids running out of privacy budget as long as the database grows fast enough.
\Blockcomposition splits the data stream into {\em blocks}, for example by time (e.g., one day's worth of data goes into one block) or by users (e.g., each user's data goes into one block), depending on the unit of protection (event- or user-level privacy).
\Blockcomposition lets training pipelines combine available blocks into larger datasets to train models effectively, but accounts for the privacy loss of releasing a model at the level of the specific blocks used to train that model.
When the privacy loss for a given block reaches a pre-configured ceiling, the block is {\em retired} and will not be used again.
However, new blocks from the stream arrive with zero privacy loss and can be used to train future models.
Thus, as long as the database adds new blocks fast enough relative to the rate at which models arrive, \sysname will never run out of privacy budget for the stream.
Finally, \blockcomposition allows adaptivity in the choice of training computation, privacy parameters, and  blocks to execute on, thereby modeling the most comprehensive form of adaptivity in the DP literature, including ~\cite{rogers2016privacy} which only considers the first two  choices.

To address the privacy-utility tradeoff we develop {\em \iterativetraining}, a training procedure that controls the utility of DP-trained models by repeatedly and adaptively training them on growing data and/or DP parameters available from the stream.
Models retrain until, with high probability, they meet programmer-specified quality criteria (e.g. an accuracy target).
\Iterativetraining uses block composition's support for adaptivity and integrates well with TFX's design, which includes a model validation stage in training pipelines.

\section{Background}
\label{sec:background}

Our effort builds upon an opportunity we observe in today's companies: the rise of {\em ML platforms}, trusted infrastructures that provide key services for ML workloads in production, plus strong library support for their development.
They can be thought of as {\em operating systems} for ML workloads.
Google has TensorFlow-Extended (TFX)~\cite{baylor2017tfx}; Facebook has FBLearner~\cite{hazelwood2018applied}; Uber has Michelangelo~\cite{Li:michelangelo}; and Twitter has DeepBird~\cite{leonardTwitterMeetsTensorFlow}.
The opportunity is to {\em incorporate DP into these platforms as a new type of access control that constrains data leakage through the models a company disseminates}.

\begin{figure}[t]
\centering
	\includegraphics[width=\linewidth]{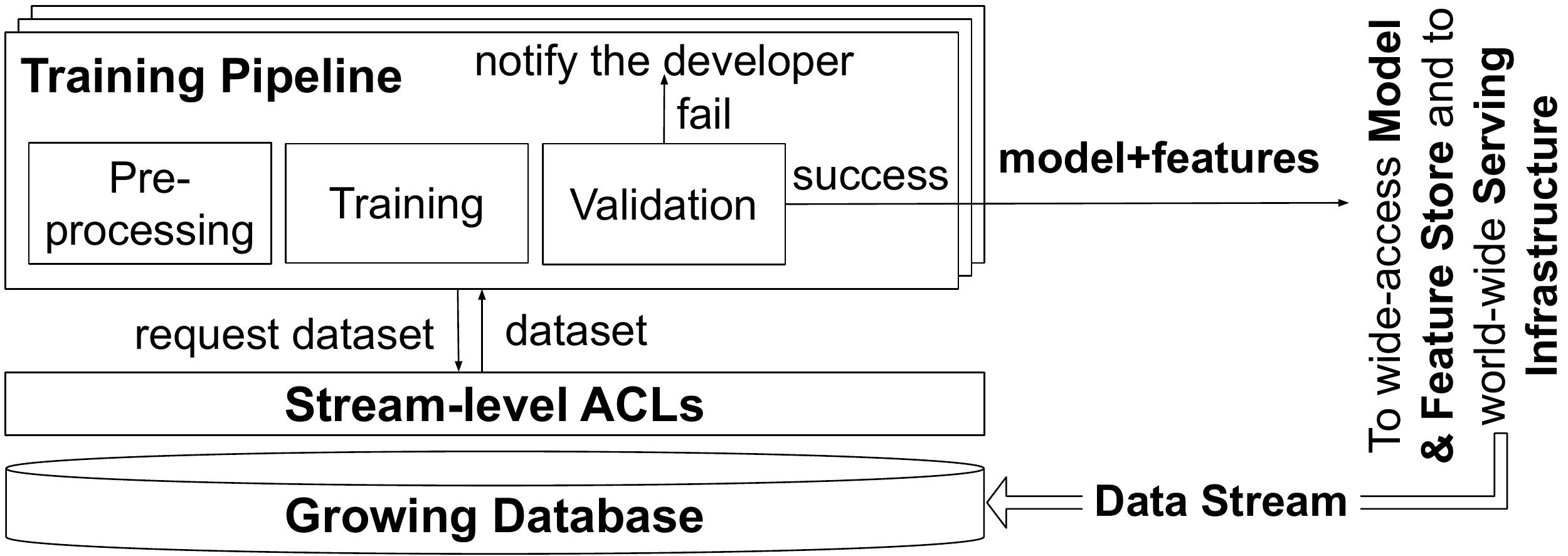}
	\vspace{-0.7cm}
	\caption{{\bf Typical Architecture of an ML Platform.} 
}
	\vspace{-0.5cm}
	\label{fig:tfx-architecture}
\end{figure}

\subsection{ML Platforms}
\label{sec:ml-platforms}

\F\ref{fig:tfx-architecture} shows our view of an ML platform; it is based on~\cite{baylor2017tfx,hazelwood2018applied,Li:michelangelo}.
The platform has several components: {\em Training Pipelines} (one for each model pushed into production), {\em Serving Infrastructure}, and a shared data store, which we call the {\em Growing Database} because it accumulates data from the company's data streams.
The access control policies on the Growing Database are exercised through {\em Stream-level ACLs} and are typically restrictive for sensitive streams.

The {\em Training Pipeline} trains a model on data from the Growing Database and verifies that it meets specific quality criteria before it is deployed for serving or shared with other teams.
It is launched periodically (e.g., daily or weekly) on datasets containing samples from a representative time window (e.g., logs over the past month).  It has three customizable modules: (1) {\em Pre-processing} loads the dataset from the Growing Database, transforms it into a format suitable for training and inference by applying feature transformation operators, and splits the transformed dataset into a {\em training set} and a {\em testing set}; (2) {\em Training} trains the model on a training set; and (3) {\em Validation} evaluates one or more {\em quality metrics} -- such as accuracy for classification or mean squared error (MSE) for regression -- on the testing set.
Validation checks that the metrics reach specific {\em quality targets} to warrant the model's release.
The targets can be fixed by developers or can be values achieved by a previous model.
If the model meets all quality criteria, it is bundled with its feature transformation operators (a.k.a. {\em features}) and pushed into serving.

The {\em Serving Infrastructure} manages the online aspects of the model.  It distributes the model+features to inference servers around the world and to end-user devices and continuously evaluates and partially updates it on new data.
The model+features bundle is also often pushed into a company-wide {\em Model \& Feature Store}, from where other teams within the company can discover it and integrate into their own models.
Twitter and Uber report sharing embedding models~\cite{twitter-embeddings} and tens of thousands of summary statistics~\cite{Li:michelangelo} across teams through their Feature Stores.
To enable such wide sharing, companies sometimes enforce more permissive access control policies on the Model \& Feature Store than on the data itself.

\subsection{Threat Model}
\label{sec:threat-model}

We are concerned with the increase in sensitive data exposure that is caused by applying looser access controls to models+features than are typically applied to data.
This includes placing models+features in company-wide Model \& Feature Stores, where they can be accessed by developers not authorized to access the data.
It includes pushing models+features to end-user devices and prediction servers that could be compromised by hackers or oppressive governments.
And it includes releasing predictions from these models to the world -- either as part of applications or through prediction APIs -- which can be used to infer specifics about training data~\cite{tramer2016stealing,shokri2017membership}.
Our goal is to ``neutralize'' the wider exposure of models+features by making the process of generating them DP across all models+features ever released from a stream.

We assume the following components are trusted and implemented correctly: Growing Database;  Stream-level ACLs; the ML platform code running a Training Pipeline.
We also {\em trust the developer} that instantiates the modules in each pipeline {\em as long as the developer is authorized} by Stream-level ACLs to access the data stream(s) used by the pipeline.
However, we do not trust the wide-access Model \& Feature Store or the locations to which the serving infrastructure disseminates the model+features or their predictions.
Once a model/feature is pushed to those components, it is considered {\em released} to the untrusted domain and accessible to adversaries.

We focus on two classes of attacks against models and statistics (see Dwork~\cite{dwork2017exposed}): (1) {\em membership inference}, in which the adversary infers whether a particular entry is in the training set based on either white-box or black-box access to the model, features, and/or predictions~\cite{backes2016membership,dwork2015robustTraceability,homer2008resolving,shokri2017membership}; and (2) {\em reconstruction attacks}, in which the adversary infers unknown sensitive attributes about entries in the training set based on similar white-box or black-box access~\cite{carlini2018theSecretSharer,dinurNissim2003revealing,dwork2017exposed}.

\subsection{Differential Privacy}
\label{sec:dp-background}

DP is concerned with whether the output of a computation over a dataset -- such as training an ML model -- can reveal information about individual entries in the dataset.  To prevent such information leakage, {\em randomness} is introduced into the computation to hide details of individual entries.
\begin{fdefinition}[Differential Privacy (DP)~\cite{dwork2014algorithmic}]
	\label{def:dp}
	A randomized algorithm $\M : \D \rightarrow \Y$ is $(\epsilon, \delta)$-DP if for any
	$\D, \D'$ with $|D \oplus D'| \leq 1$ and for any $\ES \subseteq \Y$, we have:
$P(\M(\D) \in \ES) \leq e^\epsilon P(\M(\D') \in \ES) + \delta$.
\end{fdefinition}

The $\e>0$ and $\delta \in [0,1]$ parameters quantify the strength of the privacy guarantee: small values imply that one draw from such an algorithm's output gives little information about whether it ran on $D$ or $D'$.
The {\em privacy budget} $\e$ upper bounds an $\ed$-DP computation's privacy loss with probability (1-$\delta$).
$\oplus$ is a dataset distance (e.g. the symmetric difference~\cite{Mcsherry:pinq}). If $|D \oplus D'| \leq 1$, $D$ and $D'$ are {\em neighboring datasets}.

Multiple mechanisms exist to make a computation DP.
They add noise to the computation scaled by its sensitivity $s$, the maximum change in the computation's output when run on any two neighboring datasets.
Adding noise from a Laplace distribution with mean zero and scale $\frac{s}{\epsilon}$ (denoted $laplace(0, \frac{s}{\epsilon}$)) gives $(\e,0)$-DP.
Adding noise from a Gaussian distribution scaled by $\frac{s}{\epsilon}\sqrt{2\ln(\frac{1.25}{\delta})}$ gives $\ed$-DP.

DP is known to address the attacks in our threat model~\cite{shokri2017membership,
dwork2017exposed,carlini2018theSecretSharer,236254}.
At a high level, membership and reconstruction attacks work by finding data
points that make the observed model more likely: if those points were in the
training set, the likelihood of the observed output increases.  DP prevents
these attacks, as no specific data point can drastically increase
the likelihood of the model outputted by the training procedure.

DP literature is very rich and mature, including in ML.
DP versions exist for almost every popular ML algorithm, including: stochastic gradient descent (SGD)~\cite{abadi2016deep,yu2019differentially};
various regressions~\cite{chaudhuri2008privacy, nikolaenko2013privacy, talwar2015nearly-optimal, zhang2012functional,kifer2012convexERM};
collaborative filtering~\cite{mcsherry2009differentially}; 
language models~\cite{McMahan2018LearningDP};
feature selection~\cite{chaudhuri2013nearly-optimal}; model selection~\cite{smith2013DPModelSelection}; evaluation~\cite{boyd2015differential}; and statistics, e.g. contingency tables~\cite{barak2007privacy}, histograms~\cite{xu2012differentially}.
The privacy module in TensorFlow~v2 implements several SGD-based algorithms~\cite{mcmahan2018aGeneral}.

A key strength of DP is its {\em composition} property, which in its basic form, states that the process of running an $(\e_1,\d_1)$-DP and an $(\e_2,\d_2)$-DP computation on the same dataset is $(\e_1+\e_2,\d_1+\d_2)$-DP.  Composition enables the development of complex DP computations  -- such as DP Training Pipelines -- from piecemeal DP components, such as DP ML algorithms.
Composition also lets one account for the privacy loss resulting from a sequence of DP-computed outputs, such as the release of multiple models+features.

A distinction exists between {\em user-level} and {\em event-level} privacy.
User-level privacy enforces DP on all data points contributed by a user toward a computation.
Event-level privacy enforces DP on individual data points (e.g., individual clicks).
User-level privacy is more meaningful than event-level privacy, but much more challenging to sustain on streams.
Although \sysname's design can in theory be applied to user-level privacy (\S\ref{sec:applications-to-user-level-privacy}), we focus most of the paper on {\em event-level privacy}, which we deem practical enough to be deployed in big companies.
\S\ref{sec:conclusions} discusses the limitations of this semantic.

\section{\sysname Architecture}
\label{sec:architecture}

\begin{figure}[t]
\includegraphics[width=\linewidth]{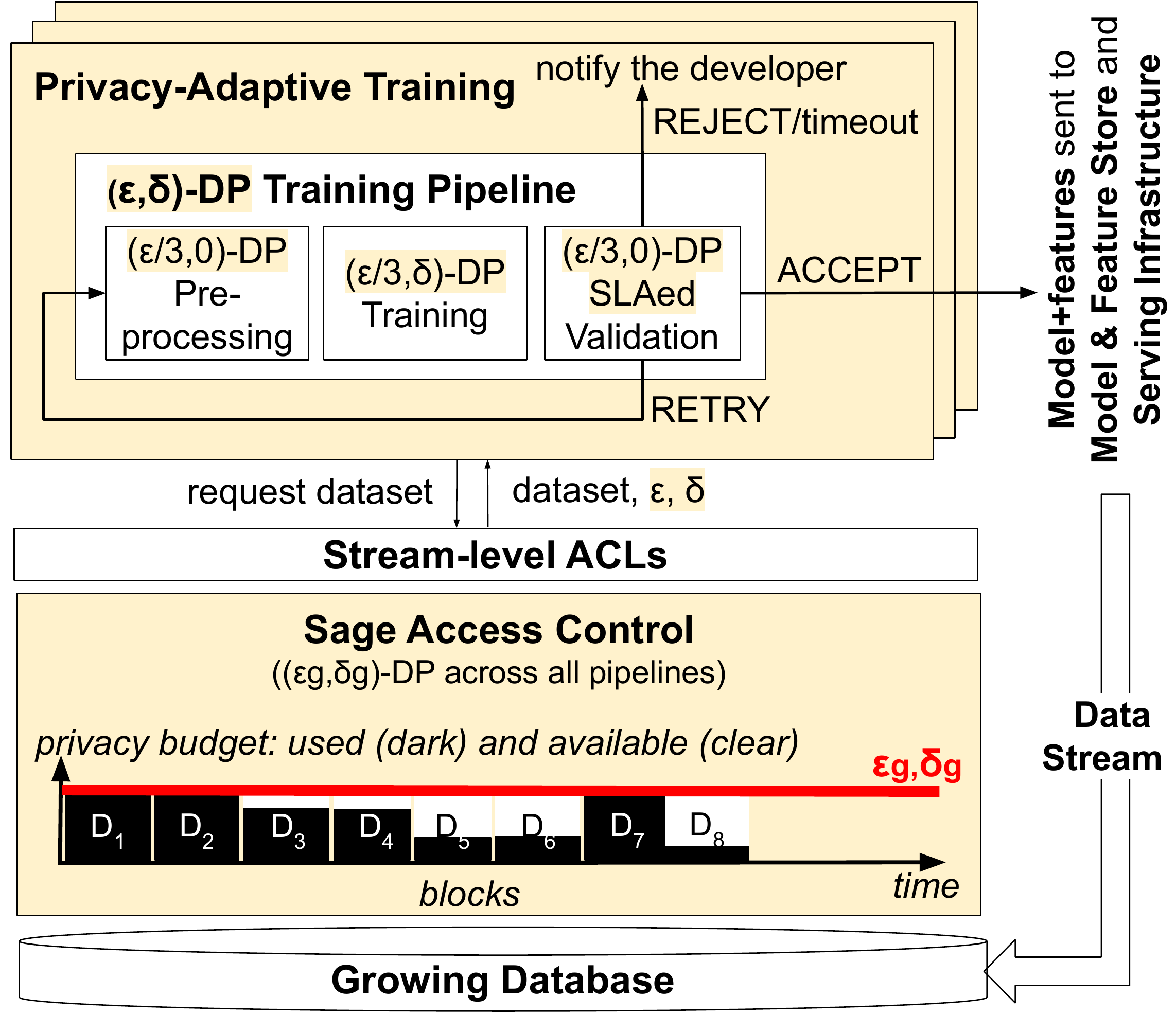}
	\vspace{-0.5cm}
	\caption{{\bf \sysname DP ML Platform.}  Highlights changes from non-DP version.
}
	\label{fig:sage-architecture}
	\vspace{-0.4cm}
\end{figure}

The \sysname training platform enforces a global $\egdg$-DP semantic over all models+features that have been, or will ever be, released from each sensitive data stream.
The highlighted portions in \F\ref{fig:sage-architecture} show the changes \sysname brings to a typical ML platform.
First, each Training Pipeline must be made to individually satisfy $\ed$-DP for some privacy parameters given by \sysname at runtime (box {\em $\ed$-DP Training Pipeline}, \S\ref{sec:dp-training-pipeline}).
The developer is responsible for making this switch to DP, and while research is needed to ease DP programming, this paper leaves that challenge aside.

Second, \sysname introduces an additional layer of access control beyond traditional stream-level ACLs (box {\em \sysname Access Control}, \S\ref{sec:sage-access-control}).
The new layer splits the data stream into {\em blocks} and accounts for the privacy loss of releasing a model+features bundle at the level of the specific blocks that were used to train that bundle.
In theory, blocks can be defined by any insensitive attribute, with two attributes particularly relevant here: time (e.g., one day's worth of data goes into one block) and user ID (e.g., all of a user's data goes into one block).
Defining blocks by time provides event-level privacy; defining them by user ID accommodates user-level privacy.
Because of our focus on the former semantic, this section assumes that blocks are defined by time; \S\ref{sec:applications-to-user-level-privacy} discusses sharding by user ID and other attributes.

When the privacy loss for a block reaches the $\egdg$ ceiling, the block is {\em retired} (blocks $D_1,D_2,D_7$ are retired in \F\ref{fig:sage-architecture}).
However, new blocks arrive with a clean budget and can be used to train future models.
Thus, {\em as long as the database grows fast enough in new blocks}, the system will never run out of privacy budget for the stream.
Perhaps surprisingly, this privacy loss accounting method, which we call {\em \blockcomposition}, is the first practical approach to avoid running out of privacy budget while enabling effective training of ML models on growing databases.
\S\ref{sec:sage-access-control} gives the intuition of \blockcomposition; \S\ref{sec:block-composition} formalizes it and proves it $\egdg$-DP.

Third, \sysname provides developers with control over the quality of models produced by the DP Training Pipelines.
Such pipelines can produce less accurate models that fail to meet their quality targets more often than without DP.  DP pipelines can also push in production low-quality models whose validations succeed by mere chance.
Both situations lead to operational headaches: the former gives more frequent notifications of failed training, the latter gives dissatisfied users.
The issue is often referred to as the {\em privacy-utility tradeoff} of running under a DP regime.
\sysname addresses this challenge by wrapping the $\ed$-DP Training Pipeline into an adaptive process that invokes training pipelines repeatedly on increasing amounts of data and/or privacy budgets to reduce the effects of DP randomness until with high probability models reach their quality criteria (box {\em \IterativeTraining}, \S\ref{sec:sage-iterative-training}).

\subsection{Example $\ed$-DP Training Pipeline}
\label{sec:dp-training-pipeline}

\begin{lstlisting}[float,floatplacement=tbp,
    label={list:example_pipeline},
    language=Python,
    caption={{\bf Example Training Pipeline.} Non-DP TFX (stricken through) and DP \sysname (highlighted) versions. TFX API simplified for exposition.}]
def preprocessing_fn(inputs, epsilon):
  dist_01 = tft.scale_to_0_1(inputs["distance"],`\hl{0,100})`
  speed_01 = tft.scale_to_0_1(inputs["speed"],`\hl{0,100}`)
  hour_of_day_speed = `\sout{group\_by\_mean}`
  `\hl{sage.dp\_group\_by\_mean}`(
    inputs["hour_of_day"], speed_01, 24, `\hl{epsilon, 1.0}`)
  return {"dist_scaled": dist_01,
    "hour_of_day": inputs["hour_of_day"],
    "hour_of_day_speed": hour_of_day_speed,
    "duration": inputs["duration"]}

def trainer_fn(hparams, schema, `\hl{epsilon, delta})`: [...]
  feature_columns = [numeric_column("dist_scaled"),
    numeric_column("hour_of_day_speed"),
    categorical_column("hour_of_day", num_buckets=24)]
  estimator = \
    `\sout{tf.estimator.DNNRegressor}``\hl{sage.DPDNNRegressor}`(
      config=run_config,
      feature_columns=feature_columns,
      dnn_hidden_units=hparams.hidden_units,
      `\hl{privacy\_budget=(epsilon, delta)}`)
  return tfx.executors.TrainingSpec(estimator,...)

def validator_fn(`\hl{epsilon})`:
  model_validator = \
    `\sout{tfx.components.ModelValidator}``\hl{sage.DPModelValidator}`(
      examples=examples_gen.outputs.output,
      model=trainer.outputs.output,
      metric_fn = _MSE_FN, target = _MSE_TARGET,
      `\hl{epsilon=epsilon, confidence=0.95, B=1}`)
  return model_validator

def dp_group_by_mean(key_tensor, value_tensor, nkeys,
  `\hl{epsilon, value\_range}`):
  key_tensor = tf.dtypes.cast(key_tensor, tf.int64)
  ones = tf.fill(tf.shape(key_tensor), 1.0)
  dp_counts = group_by_sum(key_tensor, ones, nkeys)\
    `\hl{+ laplace(0.0, 2/epsilon, nkeys)}`
  dp_sums = group_by_sum(
    key_tensor,value_tensor,nkeys)\
    `\hl{+ laplace(0.0, value\_range * 2/epsilon, nkeys)}`
  return tf.gather(dp_sums/dp_counts, key_tensor)
\end{lstlisting}

\sysname expects each pipeline submitted by the ML developer to satisfy a parameterized $\ed$-DP.
Acknowledging that DP programming abstractions warrant further research,
\L~\ref{list:example_pipeline} illustrates the changes a developer would have to make at present to convert a non-DP training pipeline written for TFX to a DP training pipeline suitable for \sysname.
Removed/replaced code is stricken through and the added code is highlighted.
The pipeline processes New York City Yellow Cab data~\cite{yellowCabData} and trains a model to predict the duration of a ride.  

To integrate with TFX (non-DP version), the developer implements three TFX callbacks.
(1)~\texttt{preprocessing\_fn} uses the dataset to compute aggregate features and make user-specified feature transformations.  The example model has three features: the distance of the ride; the hour of the day; and an aggregate feature representing the average speed of cab rides each hour of the day.
(2)~\texttt{trainer\_fn} specifies the model: it configures the columns to be modeled, defines hyperparameters, and specifies the dataset.
The example model trains with a neural network regressor.
(3)~\texttt{validator\_fn} validates the model by comparing test set MSE to a constant.

To integrate with \sysname (DP version), the developer: (a) switches library calls to DP versions of the functions (which ideally would be available in the DP ML platform) and (b) splits the $\ed$ parameters, which are assigned by \sysname at runtime, across the DP function calls.
(1) \texttt{preprocessing\_fn} replaces one call with a DP version from \sysname: the mean speed per day uses \sysname's \texttt{dp\_group\_by\_mean}.
This function (lines 33-42) computes the number of times each key appears and the sum of the values associated with each key.
It makes both DP by adding draws from appropriately-scaled Laplace distributions to each count. Each data point has exactly one key value so the privacy budget usage composes in parallel across keys~\cite{Mcsherry:pinq}.
The privacy budget is split across the sum and count queries.
We envision common functions like this being available in the DP ML platform.
(2) \texttt{trainer\_fn} switches the call to the non-private regressor with the DP implementation, which in \sysname is a simple wrapper around TensorFlow's DP SGD-based optimizer.
(3) \texttt{validator\_fn} invokes \sysname's DP model validator (\S\ref{sec:sage-iterative-training}).

\subsection{\sysname Access Control}
\label{sec:sage-access-control}
\begin{figure}[t]
	\centering
\includegraphics[width=.9\linewidth]{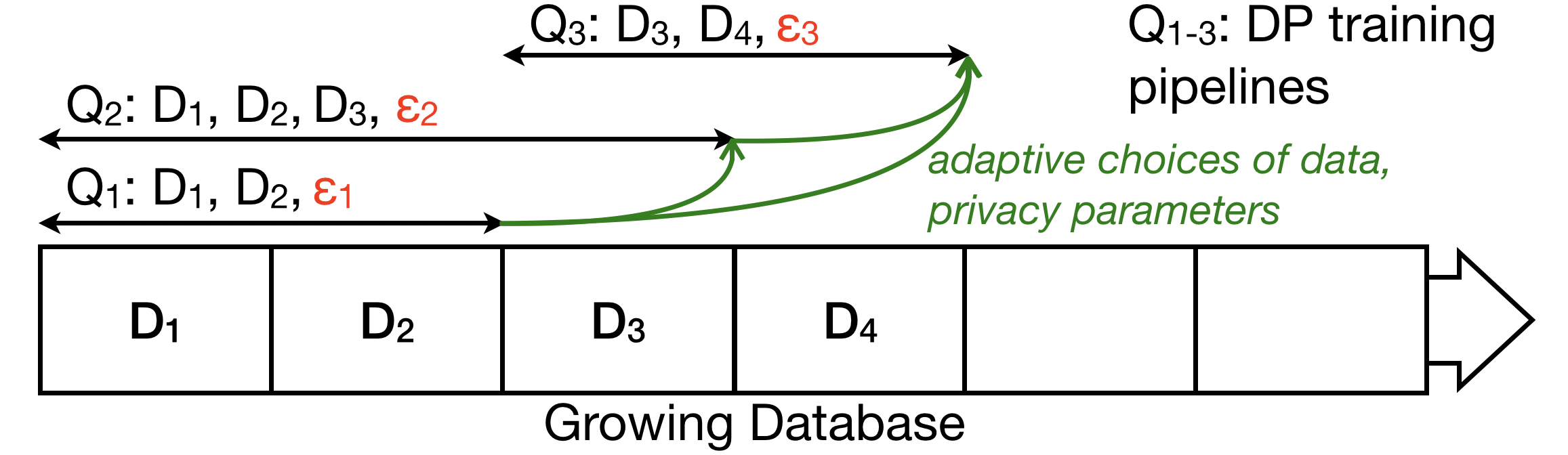}
	\vspace{-0.4cm}
	\caption{{\bf Characteristics of Data Interaction in ML.}
	}
	\label{fig:block-queries}
	\vspace{-0.4cm}
\end{figure}

\sysname uses the composition property of DP to rigorously account for (and bound) the cumulative leakage of data from sensitive data streams across multiple releases of models+features learned from these streams.
Specifically, for each sensitive stream, \sysname maintains a pre-specified, event-level $\egdg$-DP guarantee across all uses of the stream.
Unfortunately, traditional DP composition theory considers either: (a) a static database regime, whose adaptation to growing databases leads to wasteful privacy loss accounting; or (b) purely online streaming, which is inefficient for many ML workloads, including deep neural network training.
We thus developed our own composition theory, called {\em block composition}, which leverages characteristics of ML workloads running on growing databases to permit both efficient privacy loss accounting and efficient learning.
\S\ref{sec:block-composition} formalizes the new theory.
This section describes the limitations of existing DP composition for ML and gives the intuition for block composition and how \sysname uses it as a new form of access control in ML platforms.

\heading{Characteristics of Data Interaction in ML.}   \F\ref{fig:block-queries} shows an example of a typical workload as seen by an ML platform.
Each training pipeline, or {\em query} in DP parlance, is denoted $Q_i$.
We note three characteristics.
First, a typical ML workload consists of multiple training pipelines, training over time on overlapping data subsets of various sizes from an ever-growing database.
For instance, $Q_2$ may train a large deep neural network requiring massive amounts of data to reach good performance, while $Q_3$ may train a linear model with smaller data requirements, or even a simple statistic like the mean of a feature over the past day.
The pipelines are typically updated or retrained as new data is collected, with old data eventually deemed irrelevant and ignored.

Second, the data given to a training pipeline -- and for a DP model, its DP parameters -- are typically {\em chosen adaptively}.
Suppose the model trained in $Q_1$ on data from $D_{1,2}$ with budget $\epsilon_1$ gives unsatisfactory performance.
After a new block $D_3$ is collected, a developer may decide to retrain the same model in query $Q_2$ on data from $D_{1,2,3}$, and with a higher DP budget $\epsilon_2$.
Adaptivity can also happen indirectly through the data.
Suppose $Q_2$ successfully trained a recommendation model. Then, future data collected from the users (e.g., in $D_4$) may depend on the recommendations. Any subsequent query, such as $Q_3$, is potentially influenced by $Q_2$'s output.

Third, ML workloads are dynamic and continuous.
New models are introduced and expected to be trained within a reasonable amount of time; some models are run periodically with new data, others are removed from production; and this dynamic workload continues as long as new data is collected.

These characteristics imply three requirements for a composition theory suitable for ML. It must support:
\begin{enumerate}
	\item[{\bf R1}] Queries on overlapping data subsets of diverse sizes.
	\item[{\bf R2}] Adaptivity in the choice of: queries, DP parameters, and data subsets the queries process.
	\item[{\bf R3}] Endless execution on a growing database.
\end{enumerate}

\heading{Limitations of Existing Composition Theory for ML.}
No previous DP composition theory supports all three requirements.
DP has mostly been studied for static databases, where (adaptively chosen) queries are assumed to compute over {\em the entire database}.  Consequently, composition accounting is typically made at {\em query level}: each query consumes part of the total available privacy budget for the database.
Query-level accounting has carried over even in extensions to DP theory that handle streaming databases~\cite{dwork2010differential} and partitioned queries~\cite{Mcsherry:pinq}.
There are multiple ways to apply query-level accounting to ML, but each trades off at least one requirement.

First, one can query overlapping data subsets ({\bf R1}) and allow adaptivity across these queries ({\bf R2}) by accounting for composition at the query level {\em against the entire stream}.  In \F\ref{fig:block-queries}, after executing queries $Q_{1-3}$, the total privacy loss for the stream would be $\epsilon_1 + \epsilon_2 + \epsilon_3$.
This approach wastes privacy budget and leads to the problem of ``running out of budget.''
Once $\eg = \epsilon_1 + \epsilon_2 + \epsilon_3$, enforcing a global leakage bound of $\eg$ means that {\em one must stop using the stream} after query $Q_3$.
This is true even though (1) not all queries run on all the data and (2) there will be new data in the future (e.g., $D_5$).
This violates requirement ({\bf R3}) of endless execution on streams.

Second, one can restructure the queries to enable finer granularity with query-level accounting.
The data stream is partitioned in blocks, as in \F\ref{fig:block-queries}.
Each query is split into multiple sub-queries, each running DP on an individual block.
The DP results are then aggregated, for instance by averaging model updates as in federated learning~\cite{McMahan2018LearningDP}.
Since each block is a separate dataset, traditional composition can account for privacy loss at the block level.
This approach supports adaptivity ({\bf R2}) and endless execution on streams ({\bf R3}) as new data blocks incur no privacy loss from past queries.
However, it violates requirement ({\bf R1}), resulting in unnecessarily noisy learning~\cite{duchi2019lower,duchi2018minimax}.
Consider computing a feature average. DP requires adding noise once, after summing all values on the combined blocks.
But with independent queries over each block, we must add the same amount of noise to the sum over each block, yielding a more noisy total.
As another example, several DP training algorithms~\cite{abadi2016deep, McMahan2018LearningDP} fundamentally rely on sampling small training batches from large datasets to amplify privacy, which cannot be done without combining blocks.

Third, one can consume the data stream online using streaming DP.
A new data point is allocated to one of the waiting queries, which consumes its entire privacy budget.
Because each point is used by one query and discarded, DP holds over the entire stream.
New data can be adaptively assigned to any query ({\bf R2}) and arrives with a fresh budget ({\bf R3}).
However, queries cannot use past data or overlapping subsets, violating {\bf R1} and rendering the approach impractical for large models.

\heading{Block Composition.}
Our new composition theory meets all three requirements.
It splits the data stream into disjoint {\em blocks} (e.g., one day's worth of data for event-level privacy), forming a growing database on which queries can run on overlapping and adaptively chosen sets of blocks ({\bf R1}, {\bf R2}).
This lets pipelines combine blocks with available privacy budgets to assemble large datasets.
Despite running on overlapping data sets, our theoretical analysis (\S\ref{sec:block-composition}) shows that we can still account for the privacy loss at the level of individual blocks, namely that each query only impacts the blocks it actually uses, {\em not the entire data stream}.
In \F\ref{fig:block-queries}, the first three blocks each incur a privacy loss of $\epsilon_1 + \epsilon_2$ while the last block has $\epsilon_2 + \epsilon_3$.
The privacy loss of these three queries over the entire data stream will only be the maximum of these two values.
Moreover, when the database grows (e.g. block $D_5$ arrives), the new blocks' privacy loss is zero. The system can thus run endlessly by training new models on new data ({\bf R3}).

\heading{\sysname Access Control.}
With block composition, \sysname controls data leakage from a stream by enforcing DP on its blocks.
The company configures a desirable $\egdg$ global policy for each sensitive stream.
The \sysname Access Control component tracks the available privacy budget for each data block.
It allows access to a block until it runs out of budget, after which ML access to the block will forever be denied.
When the \sysname Iterator (described in \S\ref{sec:sage-iterative-training}) for a pipeline requests data, \sysname Access Control only offers blocks with available privacy budget.
The Iterator then determines the $\ed$ privacy parameters it would like to try for its next iteration and requests that budget from \sysname Access Control, which deducts $\ed$ from the available privacy budgets of those blocks  (assuming they are still available).
Finally, the Iterator invokes the developer-supplied DP Training Pipeline, trusting it to enforce the chosen $\ed$ privacy parameters.
\S\ref{sec:block-composition} proves that this access control policy enforces $\egdg$-DP for the stream.

The preceding operation is a DP-informed retention policy, but one can use block composition to define other access control policies.
Suppose the company is willing to assume that its developers (or user devices and prediction servers in distinct geographies) will not collude to violate its customers' privacy.
Then the company could enforce a separate $\egdg$ guarantee for each context (developer or geography) by
maintaining separate lists of per-block available budgets.

\subsection{\IterativeTraining}
\label{sec:sage-iterative-training}

\sysname gives developers control over the quality of the models it pushes into production, which can be affected by DP randomness.
We describe two techniques: (1) {\em SLAed validation} accounts for the effect of randomness in the validation process to ensure a high-probability guarantee of correct assessment (akin to a quality service level agreement, or SLA); and (2) {\em \iterativetraining} launches the $\ed$-DP Training Pipeline on increasing amounts of data from the stream, and/or with increased privacy parameters, to improve the model's quality adaptively until validation succeeds.
\Iterativetraining thus leverages adaptivity support in block composition to address DP's privacy-utility tradeoff.

\begin{lstlisting}[float,floatplacement=tbp,
label={list:dp_validator},
language=Python,
caption={\footnotesize \bf Implementation of \texttt{sage.DPLossValidator}.}]
class DPLossValidator(sage.DPModelValidator):
  def validate(loss_fn, target, epsilon, conf, B):
    if _ACCEPT_test(..., epsilon, (1-conf)/2, B):
      return ACCEPT
    if _REJECT_test(..., epsilon, (1-conf)/2, B):
      return REJECT
    return RETRY

  def _ACCEPT_test(test_labels, dp_test_predictions,
        loss_fn, target, epsilon, eta, B):
    n_test = dp_test_predictions.size()
    n_test_dp = n_test + laplace(2/epsilon)
    n_test_dp_min = n_test_dp -\
      2*log(3/(2*eta))/epsilon
    dp_test_loss = clip_by_value(loss_fn(test_labels,
      dp_test_predictions), 0, B)+laplace(2*B/epsilon)
    corrected_dp_test_loss = dp_test_loss +
      2*B*log(3/(2*eta))/epsilon
    return bernstein_upper_bound(
      corrected_dp_test_loss / n_test_dp_min,
      n_test_dp_min, eta/3, B) <= target

  def bernstein_upper_bound(loss, n, eta, B):
    return loss+sqrt(2*B*loss*log(1/eta)/n)+\
      4*log(1/eta)/n
\end{lstlisting}

\heading{SLAed DP Validation.}
\F\ref{fig:sage-architecture} shows the three possible outcomes of SLAed validation: {\ACCEPT}, {\REJECT}/timeout, and {\RETRY}.
If SLAed validation returns {\ACCEPT}, then with high probability (e.g. 95\%) the model reached its configured quality targets for prediction on new data from the same distribution.
In some cases, it is also possible to give statistical guarantees that the model will never reach a target irrespective of sample size and privacy parameters, in which case SLAed validation returns {\REJECT}.
\sysname also supports timing out a training procedure if it has run for too long. 
Finally, if SLAed validation returns {\RETRY}, it signals that more data is needed for an assessment.
Here we focus on the {\ACCEPT} and {\RETRY} outcomes and refer the reader to Appendix~\ref{appendix:sla-tests} for a discussion of {\REJECT} tests.

We have implemented SLAed validators for three classes of metrics: loss metrics (e.g. MSE, log loss), accuracy, and absolute errors of sum-based statistics such as mean and variance.
All validators follow the same logic.
First, we compute a DP version of the test quantity (e.g. MSE) on a testing set.
Second, we compute the {\em worst-case impact of DP noise} on that quantity for a given confidence probability; we call this a {\em correction for DP impact}.
For example, if we add Laplace noise with parameter $\frac{1}{\epsilon}$ to the sum of squared errors on $n$ data points, assuming that the loss is in $[0,1]$ we know that with probability $(1-\eta)$ the sum is deflated by less than $-\frac{1}{\epsilon}\ln(\frac{1}{2\eta})$, because a draw from this Laplace distribution has just an $\eta$ probability to be more negative than this value.
Third, we use known statistical concentration inequalities, also made DP and corrected for worst case noise impact, to upper bound with high probability the loss on the entire distribution.
We next detail the loss SLAed validator; \S\ref{appendix:sla-tests} describes the others.

\heading{Example: Loss SLAed Validator.}
A loss function is a measure of erroneous predictions on a dataset (so lower is better).  Examples: mean squared error for regression, log loss for classification, and minus log likelihood for Bayesian generative models.  \L\ref{list:dp_validator} shows our loss validator and details the implementation of its {\ACCEPT} test.

Denote: the DP-trained model $\fdp$; the loss function  range $[0,B]$; the target loss $\tau_{loss}$.
Lines 11-14 compute a DP estimate of the number of samples in the test set, corrected for the impact of DP noise to be a lower bound on the true value with probability $(1-\frac{\eta}{3})$.
Lines 15-18 compute a DP estimate of the loss sum, corrected for DP impact to be an upper bound on the true value with probability $(1-\frac{\eta}{3})$.
Lines 19-21 {\ACCEPT} the model if the upper bound is at most $\tau_{loss}$.
The bounds are based on a standard concentration inequality (Bernstein's inequality, Lines 24-25), which holds under very general conditions~\cite{Shalev-Shwartz:2014:UML:2621980}.
We show in \S\ref{appendix:sla-tests:loss} that the Loss {\ACCEPT} Test satisfies $(\e,0)$-DP and enjoys the following guarantee:
\begin{fproposition}[Loss {\ACCEPT} Test]\label{prop:loss-accept}
	With probability at least $(1-\eta)$, the {\ACCEPT} test returns true only if the expected loss of $\fdp$ is at most $\tau_{loss}$.
\end{fproposition}

\heading{\IterativeTraining.}
\sysname attempts to improve the quality of the model and its validation by supplying them with more data or privacy budgets so the SLAed validator can either {\ACCEPT} or {\REJECT} the model.
Several ways exist to improve a DP model's quality.
First, we can increase the dataset's size: at least in theory, it has been proven that one can compensate for the loss in accuracy due to {\em any} $\ed$-DP guarantee by increasing the training set size~\cite{kasiviswanathan2011can}.
Second, we can increase the privacy budget $\ed$ to decrease the noise added to the computation: this must be done within the available budgets of the blocks involved in the training and {\em not too aggressively}, because wasting privacy budget on one pipeline can prevent other pipelines from using those blocks.

\Iterativetraining searches for a configuration that can be either {\ACCEPT}ed or {\REJECT}ed by the SLAed validator.
We have investigated several strategies for this search.  Those that conserve privacy budget have proven the most efficient.
Every time a new block is created, its budget is divided evenly across the ML pipelines currently waiting in the system.
Allocated DP budget is reserved for the pipeline that received it, but \iterativetraining will not use all of it right away.
It will try to {\ACCEPT} using as little of the budget as possible.
When a pipeline is {\ACCEPT}ed, its remaining budget is reallocated evenly across the models still waiting in \sysname.

To conserve privacy budget, each pipeline will first train and test using a
small configurable budget ($\e_0,\d_0$), and a minimum window size for
the model's training.
On {\RETRY} from the validator, the pipeline will be retrained, making sure to
double either the privacy budget if enough allocation is available to the
Training Pipeline, or the number of samples available to the Training Pipeline
by accepting new data from the stream.
This doubling of resources ensures that when a model is {\ACCEPT}ed, the sum of
budgets used by all failed iterations is at most equal to the budget used by the final, accepted iteration. 
This final budget also overshoots the best possible budget by at most
two, since the model with half this final budget had a {\RETRY}.
Overall, the resources used by this DP budget search are thus at most four times the budget of the final model.
Evaluation \S\ref{sec:evaluation:end-to-end-experiment} shows that this conservative
strategy improves performance when multiple Training Pipelines contend for the
same blocks.

\section{Block Composition Theory}
\label{sec:block-composition}

\begin{figure}[t!]
  \footnotesize
  \centering
  \vspace{-0.2cm}
  \begin{subfigure}[t]{\linewidth}
    \begin{algorithm}[H]
      \footnotesize
      \caption*{(a) QueryCompose($\A$, $b$, $r$, $(\epsilon_i,\delta_i)_{i=1}^r$):}
      \begin{algorithmic}[1]
      \For{$i$ in $1$, $\dots$, $r$}
        \Comment{($\A$ depends on $\O^b_1, \dots, \O^b_{i-1}$ in iter.\ $i$)}
        \State $\A$ gives neighboring datasets $\D^{i,0}$ \& $\D^{i,1}$
        \State $\A$ gives $(\epsilon_i, \delta_i)$-DP $\M_i$
        \State $\A$ receives $\O^b_i = \M_i(\D^{i,b})$
      \EndFor
      \Return{$V^b = (\O^b_1, \dots, \O^b_r)$}
    \end{algorithmic}
    \end{algorithm}
    \vspace{-0.5cm}
    \caption{\footnotesize {\bf Traditional Query-level Accounting.}
}
  \label{fig:compose}
  \end{subfigure}

  \vspace{-0.3cm}

  \begin{subfigure}[t]{\linewidth}
    \begin{algorithm}[H]
      \footnotesize
      \caption*{(b) BlockCompose($\A$, $b$, $r$, $(\epsilon_i,\delta_i)_{i=1}^r$, $(\block_i)_{i=1}^r$):}
      \begin{algorithmic}[1]
      \State \hl{$\A$ gives two neighboring block datasets $\D^{0}$ and $\D^{1}$}
      \For{$i$ in $1$, $\dots$, $r$}
        \Comment{($\A$ depends on $\O^b_1, \dots, \O^b_{i-1}$ in iter.\ $i$)}
\State $\A$ gives $(\epsilon_i, \delta_i)$-DP $\M_i$
        \State $\A$ receives $\O^b_i = \M_i($\hl{$\bigcup\limits_{j \in \block_i} \D_{j}^{b}$}$)$
      \EndFor
      \Return{$V^b = (\O^b_1, \dots, \O^b_r)$}
    \end{algorithmic}
    \end{algorithm}
    \vspace{-0.5cm}
    \caption{\footnotesize {\bf Block Composition for Static Datasets.}
    	Change from query-level accounting shown in yellow background.
}
    \label{fig:block-compose}
  \end{subfigure}

  \vspace{-0.3cm}

  \begin{subfigure}[t]{\linewidth}
  	\begin{algorithm}[H]
  		\footnotesize
  		\caption*{(c) AdaptiveStreamBlockCompose($\A$, $b$, $r$, \hlc[green]{$\epsilon_g$, $\delta_g$, $\World$}):}
  		\begin{algorithmic}[1]
  			\State \hl{$\A$ gives $k$, the index of the block with the adversarially chosen change}
\For{$i$ in $1$, $\dots$, $r$}
        \Comment{($\A$ depends on $\O^b_1, \dots, \O^b_{i-1}$ in iter.\ $i$)}
\If{\hl{create new block $l$ and $l == k$}}
\State \hl{$\A$ gives neighboring blocks $\D_k^{0}$ and $\D_k^{1}$}
        \ElsIf{\hl{create new block $l$ and $l \neq k$}}
  			\State \hl{$\D_{l}^{b} = \D(\World, \O^b_1, \dots, \O^b_{i-1})$}
\EndIf
  			\State \hlc[green]{$\A$ gives $\block_i$, $(\epsilon_i, \delta_i)$, and $(\epsilon_i, \delta_i)$-DP $\M_i$}
  			\If{\hlc[green]{$\bigwedge\limits_{j \in \block_i} \text{AccessControl}^j_{\epsilon_g,\delta_g}(\epsilon_1^j, \delta_1^j, ..., \epsilon_i^j, \delta_i^j, 0, ...)$}}
  			   \State $\A$ receives $\O^b_i = \M_i(\bigcup\limits_{j \in \block_i} \D_{j}^{b})$
         \Else{$\A$ receives no-op $\O^b_i = \perp$}
\EndIf
  			\EndFor
  			\Return{$V^b = (\O^b_1, \dots, \O^b_r)$}
  		\end{algorithmic}
  	\end{algorithm}
  	\vspace{-0.5cm}
  	\caption{\footnotesize{\bf Sage Block Composition.}
  		Adds support for streams (yellow lines 1-6) and adaptive choice of blocks, privacy parameters (green lines 7-8).
}
  	\label{fig:sage-block-compose}
  \end{subfigure}

  \vspace{-0.3cm}

  \caption{{\bf Interaction Protocols for Composition Analysis.}
$\A$ is an algorithm defining the adversary's power; $b \in \{0,1\}$ denotes two hypotheses the adversary aims to distinguish; $r$ is the number of rounds; $(\epsilon_i,\delta_i)_{i=1}^r$ the DP parameters used at each round; $(\block_i)_{i=1}^r$ the blocks used at each round.
$\text{AccessControl}^j_{\epsilon_g,\delta_g}$ returns true if running $(\epsilon_i,\delta_i)$-DP query $\M_i$ on block $j$ ensures that with probability $\ge (1-\delta_g)$ the privacy loss for block $j$ is $\le \epsilon_g$. 
}
  \vspace{-0.3cm}
  \label{f:composition-algorithms}
\end{figure}

This section provides the theoretical backing for block composition, which we invent for \sysname but which we believe has broader applications (\S\ref{sec:applications-to-user-level-privacy}).
To analyze composition, one formalizes permissible interactions with the sensitive data in a {\em protocol} that facilitates the proof of the DP guarantee.
This interaction protocol makes explicit the worst-case decisions that can be made by modeling them through an adversary.
In the standard protocol (detailed shortly), an adversary $\A$ picks the neighboring data sets and supplies the DP queries that will compute over one of these data sets; the choice between the two data sets is exogenous to the interaction.
To prove that the interaction satisfies DP, one must show that given the results of the protocol, it is impossible to determine with high confidence which of the two neighboring data sets was used.

\F\ref{f:composition-algorithms} describes three different interaction protocols of increasing sophistication.
Alg.~(\ref{fig:compose}) is the basic DP composition protocol.
Alg.~(\ref{fig:block-compose}) is a block-level protocol we propose for static databases.
Alg.~(\ref{fig:sage-block-compose}) is the protocol adopted in \sysname; it extends Alg.~(\ref{fig:block-compose}) by allowing a streaming database and adaptive choices of blocks and privacy parameters.
Highlighted are changes made to the preceding protocol.

\subsection{Traditional Query-Level Accounting}
\label{sec:theory:query-level-accounting}
{\em QueryCompose} (Alg.~(\ref{fig:compose})) is the interaction protocol assumed in most analyses of composition of several DP interactions with a database.
There are three important characteristics.
First, the number of queries $r$ and the DP parameters $(\epsilon_i,\delta_i)_{i=1}^r$ are fixed in advance.
However the DP queries $\M_i$ can be chosen adaptively.
Second, the adversary adaptively chooses neighboring datasets $\D^{i,0}$ and
$\D^{i,1}$ for each query.
This flexibility lets the protocol readily support adaptively evolving data (such as with data streams) where future data collected may be impacted by the adversary's change to past data.
Third, the adversary receives the results $V^b$ of running the DP queries
$\M_i$ on $\D^{i,b}$; here, $b \in \{0,1\}$ is the exogenous choice of which database to use and is unknown to $\A$.
DP is guaranteed if $\A$ cannot confidently learn $b$ given $V^b$.

A common tool to analyze DP protocols is {\em privacy loss}:
\begin{fdefinition}[Privacy Loss]
\label{def:privacy-loss}
Fix any outcome $v = (v_1, \dotsc, v_r)$ and denote $v_{<i} = (v_1, \dotsc, v_{i-1})$. The {\em privacy loss} of
an algorithm Compose($\A$, $b$, $r$, $\cdot$) is:
\[
  \Loss(v) = \ln\Big( \frac{P(V^0 = v)}{P(V^1 = v)} \Big) = \ln\Big( \prod\limits_{i=1}^r \frac{P(\O^0_i = v_i | v_{<i})}{P(\O^1_i = v_i | v_{<i})} \Big)
\]
\end{fdefinition}
Bounding the privacy loss for any adversary $\A$
with high probability implies DP~\cite{kasiviswanathan2014semantics}.
Suppose that for any $\A$, with probability $\ge (1 - \delta)$ over draws from $v \sim V^0$, we have: $| \Loss(v) | \leq \epsilon$.
Then Compose($\A$, $b$, $r$, $\cdot$) is $(\epsilon, \delta)$-DP.
This way, privacy loss and DP are defined in terms of distinguishing
between two hypotheses indexed by $b \in \{0,1\}$.

Previous composition theorems (e.g. basic composition~\cite{Dwork:2006:CNS:2180286.2180305},
strong composition~\cite{dwork2010boosting},
and variations thereof~\cite{kairouz2015composition}) analyze Alg.~(\ref{fig:compose}) to derive various arithmetics for computing the overall DP semantic of interactions adhering to that protocol.
In particular, the basic composition theorem~\cite{Dwork:2006:CNS:2180286.2180305} proves that QueryCompose($\A$, $b$, $r$, $(\epsilon_i,\delta_i)_{i=1}^r$)
is $(\sum_{i=1}^r\epsilon_i, \sum_{i=1}^r\delta_i)$-DP.
These theorems form the basis of most ML DP work.
However, because composition is accounted for at the query level, imposing a fixed global privacy budget means that one will ``run out'' of it and stop training models even on new data.

\subsection{Block Composition for Static Datasets}
\label{sec:theory-block-composition-for-static-datasets}

Block composition improves privacy accounting for workloads where interaction consists of queries that run on {\em overlapping data subsets of diverse sizes}.
This is one of the characteristics we posit for ML workloads (requirement {\bf R1} in \S\ref{sec:sage-access-control}).
Alg.~(\ref{fig:block-compose}), {\em BlockCompose}, formalizes this type of interaction for a {\em static dataset setting} as a springboard to formalizing the full ML interaction.
We make two changes to QueryCompose.
First (line 1), the neighboring datasets are defined once and for all before
interacting. This way, training pipelines accessing non-overlapping
parts of the dataset cannot all be impacted by one entry's change.
Second (line 4), the data is split in blocks, and each DP query runs on a subset of the blocks. 

We prove that the privacy loss over the entire dataset is
the same as the maximum privacy loss on each block, accounting only for queries using this block:
\begin{ftheorem}[Reduction to Block-level Composition]
\label{theorem:block-reduction}
  The privacy loss of BlockCompose($\A$, $b$, $r$, $(\epsilon_i,\delta_i)_{i=1}^r$, $(\block_i)_{i=1}^r$)
is upper-bounded by the maximum privacy loss for any block:
{
\begin{align*}
  | \Loss(v) | \leq \max_{k} \big| \ln\Big( \prod\limits_{\substack{i=1 \\ k \in \block_i}}^r \frac{P(\O^0_i = v_i | v_{<i})}{P(\O^1_i = v_i | v_{<i})} \Big) \big| .
\end{align*}
}
\end{ftheorem}
\begin{fproof}
{Let $\D^0$ and $\D^1$ be the neighboring datasets picked by adversary $\A$, and let $k$ be the block index s.t.\
  $\D^0_l = \D^1_l$ for all $l \neq k$, and $|\D^0_{k} \oplus \D^1_{k}| \leq 1$. For any result $v$ of Alg.~(\ref{fig:block-compose}):
{\footnotesize \begin{align*}
  \big| & \Loss(v) \big| = \big| \ln \Big( \prod_{i=1}^r \frac{P(\O^0_i = v_i | v_{<i})}{P(\O^1_i = v_i | v_{<i})} \Big) \big| \\
   & = \big| \ln \Big( \prod_{\substack{i=1 \\ k \in \block_i}}^r \frac{P(\O^0_i = v_i | v_{<i})}{P(\O^1_i = v_i | v_{<i})} \Big)
   + \cancel{\ln \Big( \prod_{\substack{i=1 \\ k \notin \block_i}}^r \frac{P(\O^0_i = v_i | v_{<i})}{P(\O^1_i = v_i | v_{<i})} \Big) \big|} \\
  & \leq \max_{k} \big| \ln \Big( \prod_{\substack{i=1 \\ k \in \block_i}}^r \frac{P(\O^0_i = v_i | v_{<i})}{P(\O^1_i = v_i | v_{<i})} \Big) \big|
\end{align*}
}  The slashed term is zero because if $k \notin \block_i$, then \\ $\bigcup\limits_{j \in \block_i} \D_{j}^{0} = \bigcup\limits_{j \in \block_i} \D_{j}^{1}$, hence $\frac{P(\O^0_i = v_i | v_{<i})}{P(\O^1_i =v_i | v_{<i})} = 1$.
}
\end{fproof}

Hence, unused data blocks allow training of other (adaptively chosen) ML models, and exhausting the DP budget of a block means we retire that block of data, and not the entire data set.
This result, which can be extended to strong composition (see Appendix~\ref{a:strong-composition}), can be used to do tighter accounting than query-level accounting when the workload consists of queries on overlapping sets of data blocks (requirement {\bf R1}). However, it does not support adaptivity in block choice or a streaming setting, violating {\bf R2} and {\bf R3}.

\subsection{\sysname Block Composition}
\label{sec:theory:sage-block-composition}

Alg.~(\ref{fig:sage-block-compose}), {\em AdaptiveStreamBlockCompose}, addresses the preceding limitations with two changes.
First, supporting streams requires that datasets not be fixed before interacting, because future data depends on prior models trained and pushed into production.
The highlighted portions of lines 1-10 in Alg.~(\ref{fig:sage-block-compose}) formalize the dynamic nature of data collection by having new data blocks explicitly
depend on previously trained models, which are chosen by the adversary, in addition to other mechanisms of the world $\World$ that are not impacted by the adversary.
Fortunately, Theorem~\ref{theorem:block-reduction} still applies, because model
training can only use blocks that existed at the time of training, which in turn only depend on prior blocks through DP trained models.
Therefore, new data blocks can train new ML models, enabling endless operation on streams ({\bf R3}).

Second, interestingly, supporting adaptive choices in the data blocks implies supporting adaptive choices in the queries' DP budgets.
For a given block, one can express query $i$'s choice to use block $j$ as using a privacy budget of either $(\epsilon_i, \delta_i)$ or $(0, 0)$.
Lines 7-8 in Alg.~(\ref{fig:sage-block-compose}) formalize the adaptive choice of both privacy budgets and blocks (requirement {\bf R2}).  In supporting both, we leverage recent work on DP composition under adaptive DP budgets~\cite{rogers2016privacy}.  At each round, $\A$ requests access to a group of blocks $\block_i$, on which to run an $(\epsilon_i, \delta_i)$-DP query. \sysname's Access Control permits the query only if the privacy loss of each block in $\block_i$ will remain below $\egdg$.
Applying our Theorem~\ref{theorem:block-reduction} and \cite{rogers2016privacy}'s Theorem 3.3, we prove the following result (proof in \S\ref{a:basic-composition}):

\begin{ftheorem}[Basic Composition for Sage Block Composition]
\label{theorem:adaptive-block-basic-compose}
  AdaptiveStreamBlockCompose($\A$,$b$,$r$,$\epsilon_g$,$\delta_g$,$\World$) is $(\epsilon_g,\delta_g)$-DP if for all $k$, $\text{AccessControl}^k_{\epsilon_g,\delta_g}$ enforces:
{
\begin{equation*}
  \Big( \sum_{\substack{i=1 \\ k \in \block_i}}^r \epsilon_i(v_{<i}) \Big) \le \epsilon_g \ \text{and} \ \Big( \sum_{\substack{i=1 \\ k \in \block_i}}^r \delta_i(v_{<i}) \Big) \le \delta_g .
\end{equation*}
}
\end{ftheorem}

The implication of the theorem is that under the access control scheme described in \S\ref{sec:sage-access-control}, \sysname achieves {\em event-level $\egdg$-DP over the sensitive data stream}.
In \S\ref{a:strong-composition} we further analyze strong composition for block-level accounting.

\subsection{Defining Blocks by User ID and Other Attributes}
\label{sec:applications-to-user-level-privacy}

Block composition theory can be extended to accommodate user-level privacy and other use cases.
The theory shows that one can split a static dataset (Theorem~\ref{theorem:block-reduction}) or a data stream (Theorem~\ref{theorem:adaptive-block-basic-compose}) into disjoint blocks, and run DP queries adaptively on overlapping subsets of the blocks while accounting for privacy at the block level.
The theorems are focused on {\em time} splits, but the same theorems can be written for splits based on {\em any attribute whose possible values can be made public}, such as geography, demographics, or user IDs.
Consider a workload on a static dataset in which queries combine data from diverse and overlapping subsets of countries, e.g., they compute average salary in each country separately, but also at the level of continents and GDP-based country groups.
For such a workload, block composition gives tighter privacy accounting across these queries than traditional composition, though the system will still run out of privacy budget eventually because no new blocks appear in the static database.

As another example, splitting a stream by user ID enables querying or ignoring all observations from a given user, adding support for user-level privacy.
Splitting data over user ID requires extra care.
If existing user IDs are not knows, each query might select user IDs that do not exist yet, spending their DP budget without adding data.
However, making user IDs public can leak information.
One approach is to use incrementing user IDs (with this fact public), and periodically run a DP query computing the maximum user ID in use.
This would ensure DP, while giving an estimate of the range of user IDs that can be queried.
In such a setting, block composition enables fine-grain DP accounting over queries on any subset of the users.
While our block theory supports this use case, it suffers from a major practical challenge.
New blocks are now created only when new users join the system, so new users must be added at a high rate relative to the model release rate to avoid running out of budget.
This is unlikely to happen for mature companies, but may be possible for emerging startups or hospitals, where the stream of incoming users/patients may be high enough to sustain modest workloads.

\section{Evaluation}
\label{sec:evaluation}

We ask four questions:
({\bf Q1})~Does DP impact Training Pipeline reliability?
({\bf Q2})~Does \iterativetraining increase DP Training Pipeline reliability?
({\bf Q3})~Does block composition help over traditional composition?
({\bf Q4})~How do ML workloads perform under \sysname's $\egdg$-DP regime?

\begin{table}
	\footnotesize
	\centering
	\begin{tabular}{l|cl}
		\hline
        \multicolumn{3}{c}{{\bf Taxi Regression Task}}\\
		\hline
		{\bf Pipelines:} & \multicolumn{2}{l}{{\bf Configuration:}} \\
		\hline
		\multirow{4}{1.3cm}{Linear Regression ({\bf LR})} & DP Alg. & AdaSSP from \cite{wang2018revisiting}, $(\e,\delta)$-DP\\
		\cline{2-3}
		& Config. & Regularization param $\rho: 0.1$ \\
		\cline{2-3}
		& Budgets & $\ed \in \{(1.0, 10^{-6}), (0.05, 10^{-6})\}$ \\
		\cline{2-3}
		& Targets & MSE $\in [2.4\times10^{-3},7\times10^{-3}]$  \\
\hline
		\multirow{5}{1.3cm}{Neural Network ({\bf NN})} & DP Alg. & DP SGD from \cite{abadi2016deep}, $\ed$-DP\\
		\cline{2-3}
		&           & ReLU, 2 hidden layers (5000/100 nodes)\\
		& Config. &Learning rate: 0.01, Epochs: 3\\
		&            & Batch: 1024, Momentum: 0.9\\
		\cline{2-3}
		& Budgets & $\ed \in \{(1.0,10^{-6}),(0.5,10^{-6})\}$ \\
		\cline{2-3}
		& Targets & MSE $\in [2\times10^{-3},7\times10^{-3}]$\\
		\hline
Avg.Speed x3* & Targets & Absolute error $ \in \{1, 5, 7.5, 10, 15\}$ km/h\\
		\hline
		\multicolumn{3}{c}{{ }}\\
		\hline
        \multicolumn{3}{c}{{\bf Criteo Classification Task}}\\
		\hline
        {\bf Pipelines:} & \multicolumn{2}{l}{{\bf Configuration:}} \\
		\hline
		\multirow{4}{1.3cm}{Logistic Regression ({\bf LG})} & DP Alg. & DP SGD from \cite{mcmahan2018general}, $(\e,\delta)$-DP\\
		\cline{2-3}
		& Config. & Learning rate: $0.1$, Epochs: 3 Batch: 512\\
		\cline{2-3}
		& Budgets & $\ed \in \{(1.0, 10^{-6}), (0.25, 10^{-6})\}$ \\
		\cline{2-3}
		& Targets & Accuracy $\in [0.74, 0.78]$  \\
		\hline
\multirow{5}{1.3cm}{Neural Network ({\bf NN})} & DP Alg. & DP SGD from \cite{mcmahan2018general}, $\ed$-DP\\
		\cline{2-3}
        &         & ReLU, 2 hidden layers (1024/32 nodes) \\
		& Config. & Learning rate: 0.01, Epochs: 5\\
		&            & Batch: 1024 \\
		\cline{2-3}
		& Budgets & $\ed \in \{(1.0,10^{-6}),(0.25,10^{-6})\}$ \\
		\cline{2-3}
		& Targets & Accuracy $\in [0.74, 0.78]$  \\
		\hline
Counts x26** & Targets & Absolute error $ \in \{0.01,0.05,0.10\}$ \\
        \hline

	\end{tabular}
    \caption{{{\bf Experimental Training Pipelines.} *Three time granularities: hour of day, day of week, week of month. **Histogram of each categorical feature.}}
	\vspace{-0.5cm}
	\label{t:private-models}
\end{table}

\begin{figure*}[t!]
	\centering
	\footnotesize
	\begin{subfigure}[t]{0.245\linewidth}
		\includegraphics[width=\linewidth]{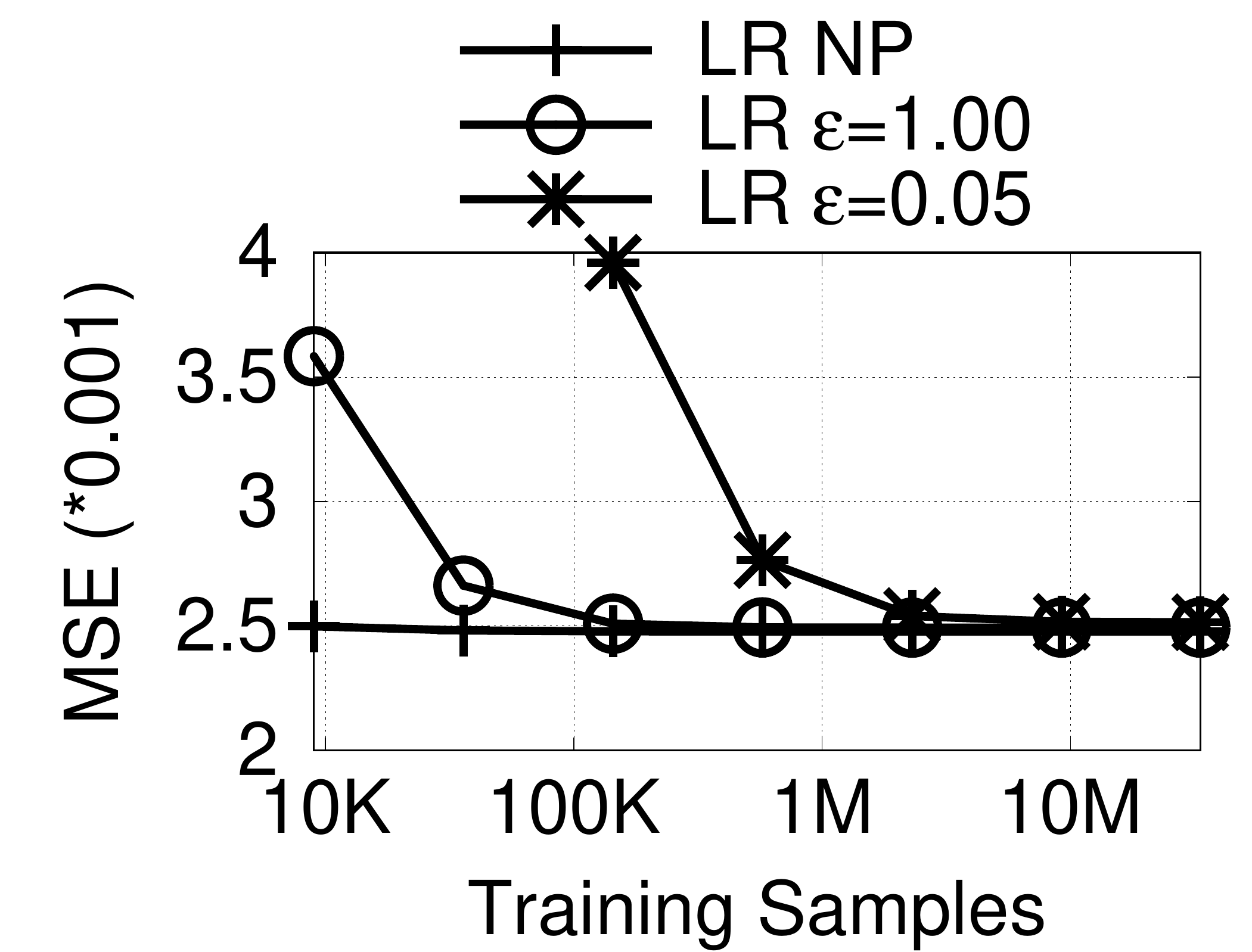}
		\caption{{\bf Taxi LR MSE}}
		\label{fig:evaluation:impact-of-dp-on-model-accuracy-taxi-linear}
	\end{subfigure}
\begin{subfigure}[t]{0.245\linewidth}
		\includegraphics[width=\linewidth]{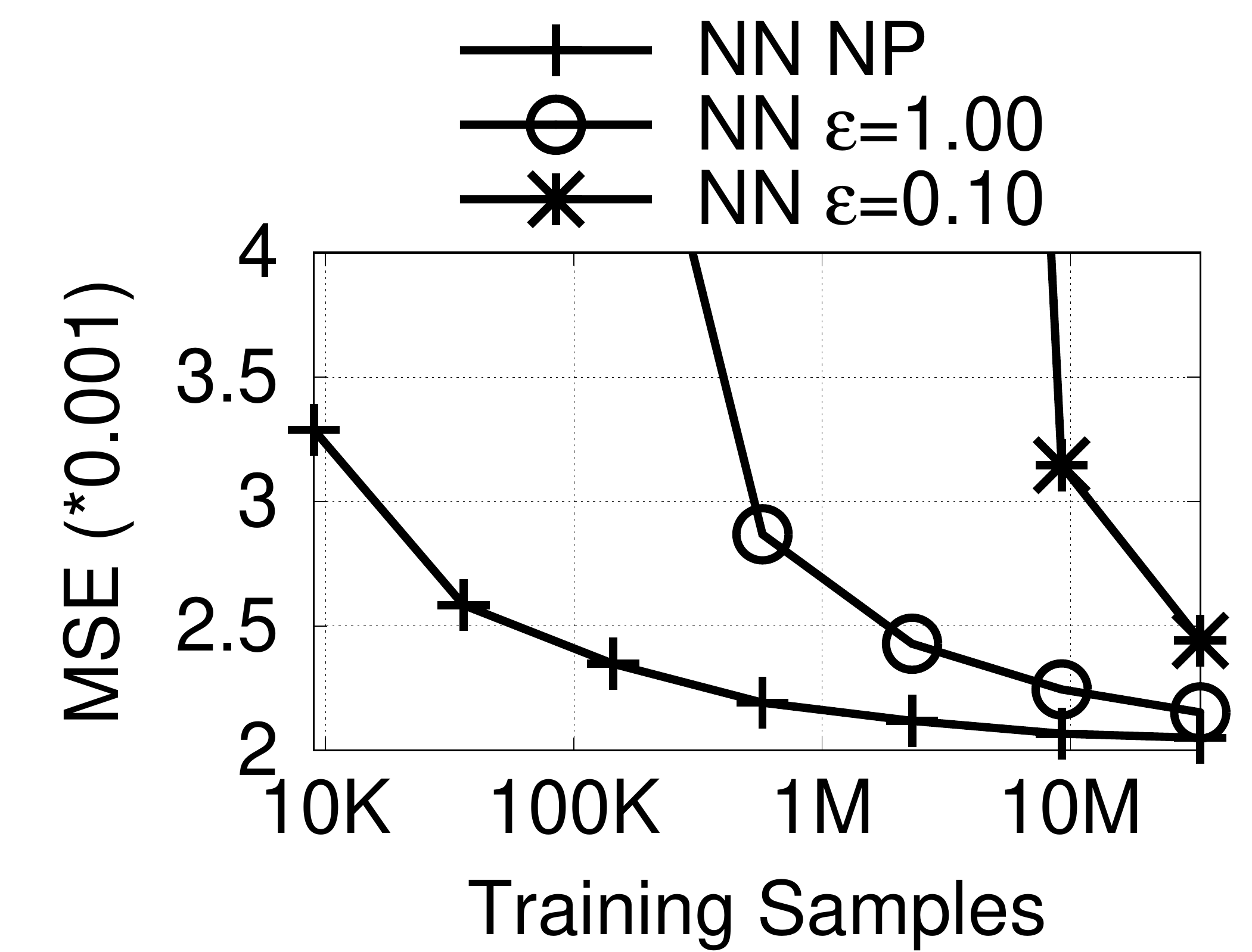}
		\caption{{\bf Taxi NN MSE}}
		\label{fig:evaluation:impact-of-dp-on-model-accuracy-taxi-nn}
	\end{subfigure}
\begin{subfigure}[t]{0.245\linewidth}
        \includegraphics[width=\linewidth]{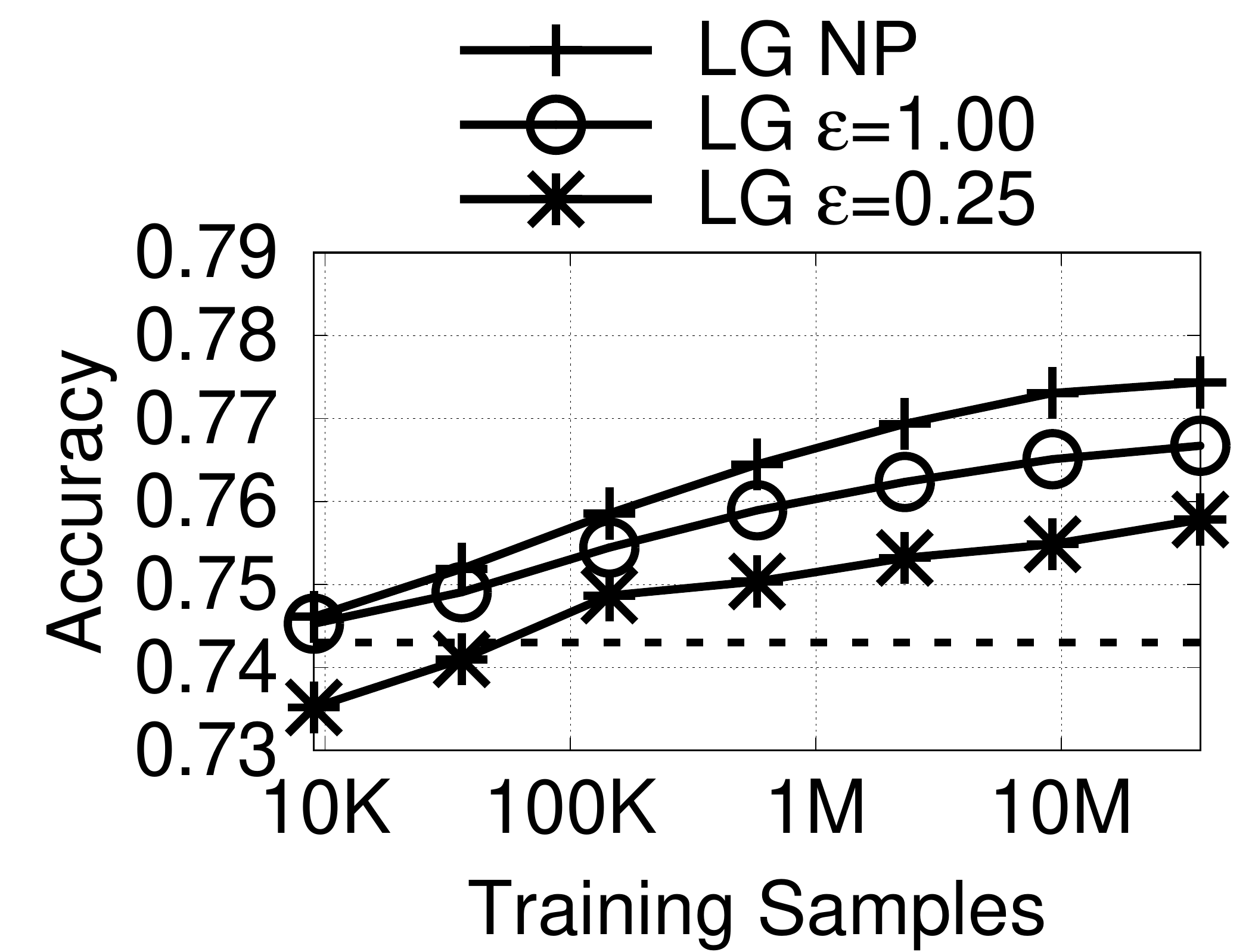}
		\caption{{\bf Criteo LG Accuracy}}
		\label{fig:evaluation:impact-of-dp-on-model-accuracy-linear-criteo-classification}	
	\end{subfigure}
  \begin{subfigure}[t]{0.245\linewidth}
        \includegraphics[width=\linewidth]{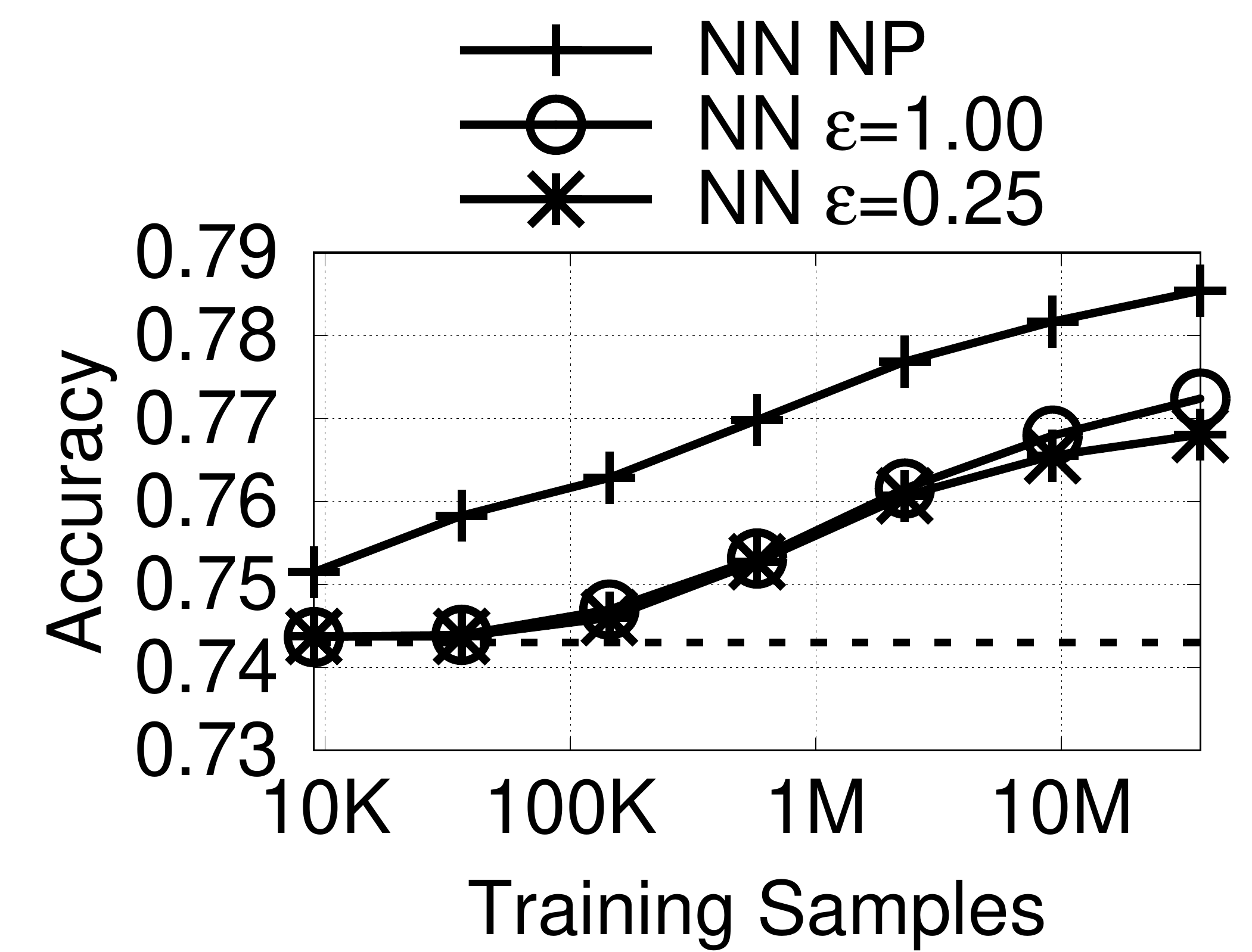}
		\caption{{\bf Criteo NN Accuracy}}
		\label{fig:evaluation:impact-of-dp-on-model-accuracy-nn-criteo-classification}
	\end{subfigure}
	\vspace{-0.4cm}
	\caption{{\bf Impacts on TFX Training Pipelines.}
        Impact of DP on the overall performance of training pipelines.
        \ref{fig:evaluation:impact-of-dp-on-model-accuracy-taxi-linear}, and
        \ref{fig:evaluation:impact-of-dp-on-model-accuracy-taxi-nn}
show the MSE loss on the Taxi regression task (lower is better).
        \ref{fig:evaluation:impact-of-dp-on-model-accuracy-linear-criteo-classification};
        \ref{fig:evaluation:impact-of-dp-on-model-accuracy-nn-criteo-classification} show the accuracy on the Criteo classification task (higher is better).
The dotted lines are na\"ive model performance.
	}  \label{fig:evaluation:impacts-of-dp-on-training-pipelines}
	\vspace{-0.3cm}
\end{figure*}

\heading{Methodology.}
We consider two datasets: 37M-samples from three months of NYC taxi rides~\cite{yellowCabData} and 45M ad impressions from Criteo~\cite{criteoKaggle}.
On the Taxi dataset we define a regression task to predict the duration of each ride using 61 binary features derived from 10 contextual features.
We implement pipelines for a {\em linear regression (LR)}, a {\em neural network (NN)}, and three statistics (average speeds at three time granularities).
On the Criteo dataset we formulate a binary classification task predicting ad clicks from 13 numeric and 26 categorical features.
We implement a {\em logistic regression (LG)}, a {\em neural network (NN)}, and histogram pipelines.
\T~\ref{t:private-models} shows details.

{\em Training:} We make each pipeline DP using known algorithms, shown in \T~\ref{t:private-models}.
{\em Validation:} We use the loss, accuracy, and absolute error SLAed validators on the regression, classification, and statistics respectively.
{\em Experiments:} Each model is assigned a quality target from a range of possible values, chosen between the best achievable model, and the performance of a na\"ive model (predicting the label mean on Taxi, with MSE $0.0069$, and the most common label on Criteo, with accuracy 74.3\%).
Most evaluation uses \iterativetraining, so privacy budgets are chosen by \sysname, with an upper-bound of $\epsilon=1$.
While no consensus exists on what a reasonable DP budget is, this value is in line with practical prior work~\cite{abadi2016deep,McMahan2018LearningDP}.
Where DP budgets must be fixed, we use values indicated in \T~\ref{t:private-models} which correspond to a large budget ($\epsilon=1$), and a small budget that varies across tasks and models.
Other defaults: 90\%::10\% train::test ratio; $\eta=0.05$; $\delta=10^{-6}$.
{\em Comparisons:} We compare \sysname's performance to existing DP composition approaches described in \S\ref{sec:sage-access-control}.
We ignore the first alternative, which violates the endless execution requirement {\bf R3} and cannot support ML workloads.
We compare with the second and third alternatives, which we call {\em query composition} and {\em streaming composition}, respectively.

\subsection{Unreliability of DP Training Pipelines in TFX (Q1)}
\label{sec:evaluation:dp-training-pipeline}

We first evaluate DP's impact on model training.
\F\ref{fig:evaluation:impacts-of-dp-on-training-pipelines} shows the loss or accuracy of each model when trained on increasing amounts data and evaluated on 100K held-out samples from their respective datasets.
Three versions are shown for each model: the non-DP version (NP), a large DP budget version ($\e=1$), and a small DP budget configuration with $\e$ values that vary across the model and task.
For both tasks, the NN requires the most data but outperforms the linear model in the private and non-private settings.
The DP LRs catch up to the non-DP version with the full dataset,
but the other models trained with SGD require more data.  Thus, model quality is impacted by DP but the impact diminishes with more training data.
This motivates \iterativetraining.

\begin{table}
	\centering
	\footnotesize
\begin{tabular}{|c|c|c|c|c|c|}
		\hline
		{\bf Dataset} & {\bf $\eta$} & {\bf No SLA} & {\bf NP SLA} & {\bf UC DP SLA} & {\bf Sage SLA}\\
		\hline
		\multirow{2}{*}{Taxi} & 0.01 & 0.379 & 0.0019 & 0.0172 & 0.0027\\
& 0.05 & 0.379 & 0.0034 & 0.0224 & 0.0051\\
		\hline
		\multirow{2}{*}{Criteo} & 0.01 & 0.2515 & 0.0052 & 0.0544    & 0.0018\\
& 0.05 & 0.2515 & 0.0065 & 0.0556    & 0.0023\\
		\hline
	\end{tabular}
	\caption{\footnotesize {\bf Target Violation Rate of {\ACCEPT}ed Models.}
		Violations are across all models separately trained with \iterativetraining.}
	\label{tab:evaluation:acceptance-violation-rate}
	\vspace{-0.7cm}
\end{table}

To evaluate DP's impact on validation, we train and validate our models for both tasks, with and without DP. We use TFX's vanilla validators, which compare the model's performance on a test set to the quality metric (MSE for taxi, accuracy for Criteo).
We then re-evaluate the models' quality metrics on a separate, 100K-sample held-out set and measure the fraction of models accepted by TFX that violate their targets on the re-evaluation set.
With non-DP pipelines (non-DP training and validation), the false acceptance rate is 5.1\% and 8.2\% for the Taxi and Criteo tasks respectively.  With DP pipelines (DP training, DP validation), false acceptance rates hike to 37.9\% and 25.2\%, motivating SLAed validation.

\begin{figure*}[t!]
  \begin{subfigure}[t]{0.245\linewidth}
	\includegraphics[width=\linewidth]{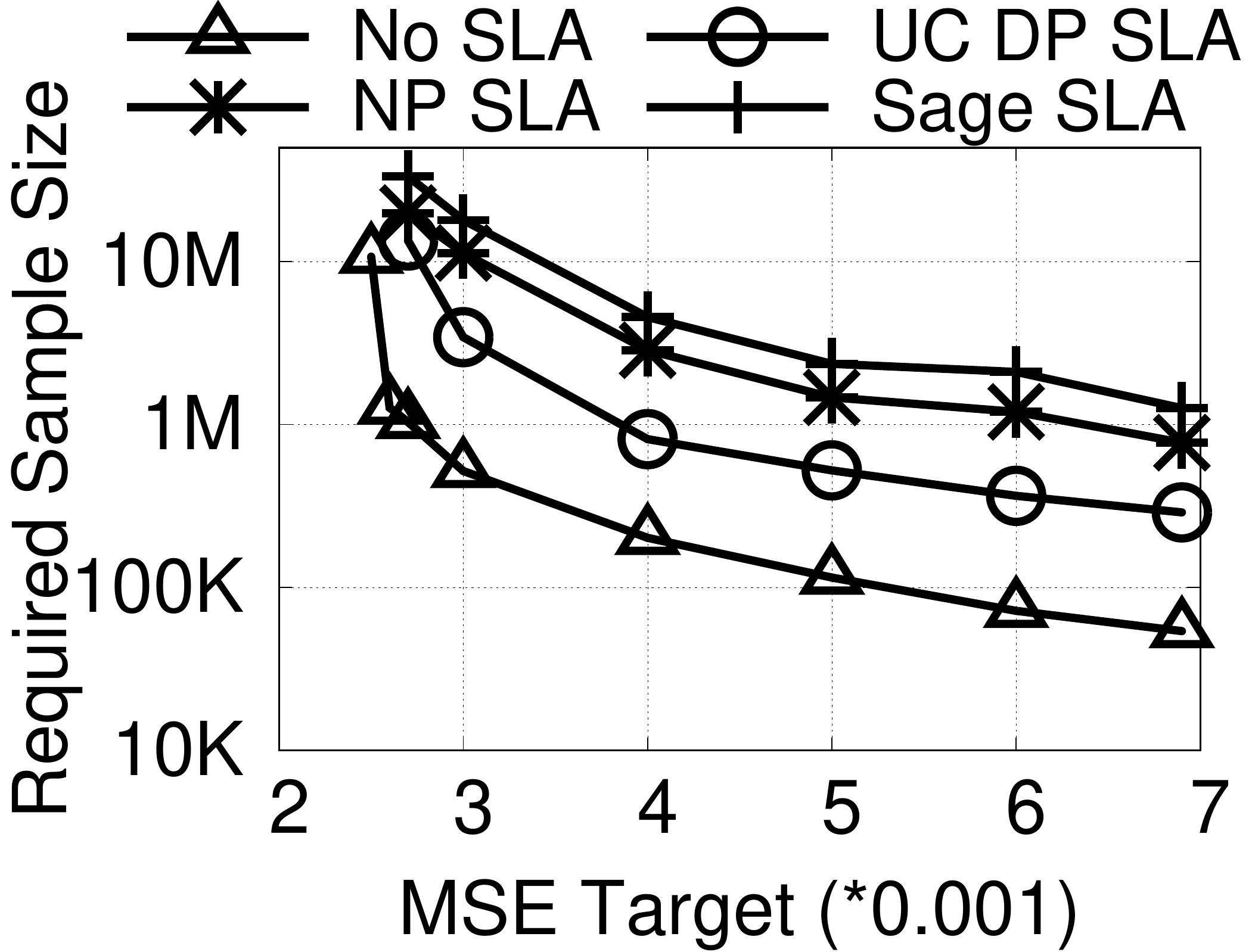}
		\caption{{\bf Taxi LR \ACCEPT}}
		\label{f:lr_required_sample_complexity}
	\end{subfigure}
  \begin{subfigure}[t]{0.245\linewidth}
		\includegraphics[width=\linewidth]{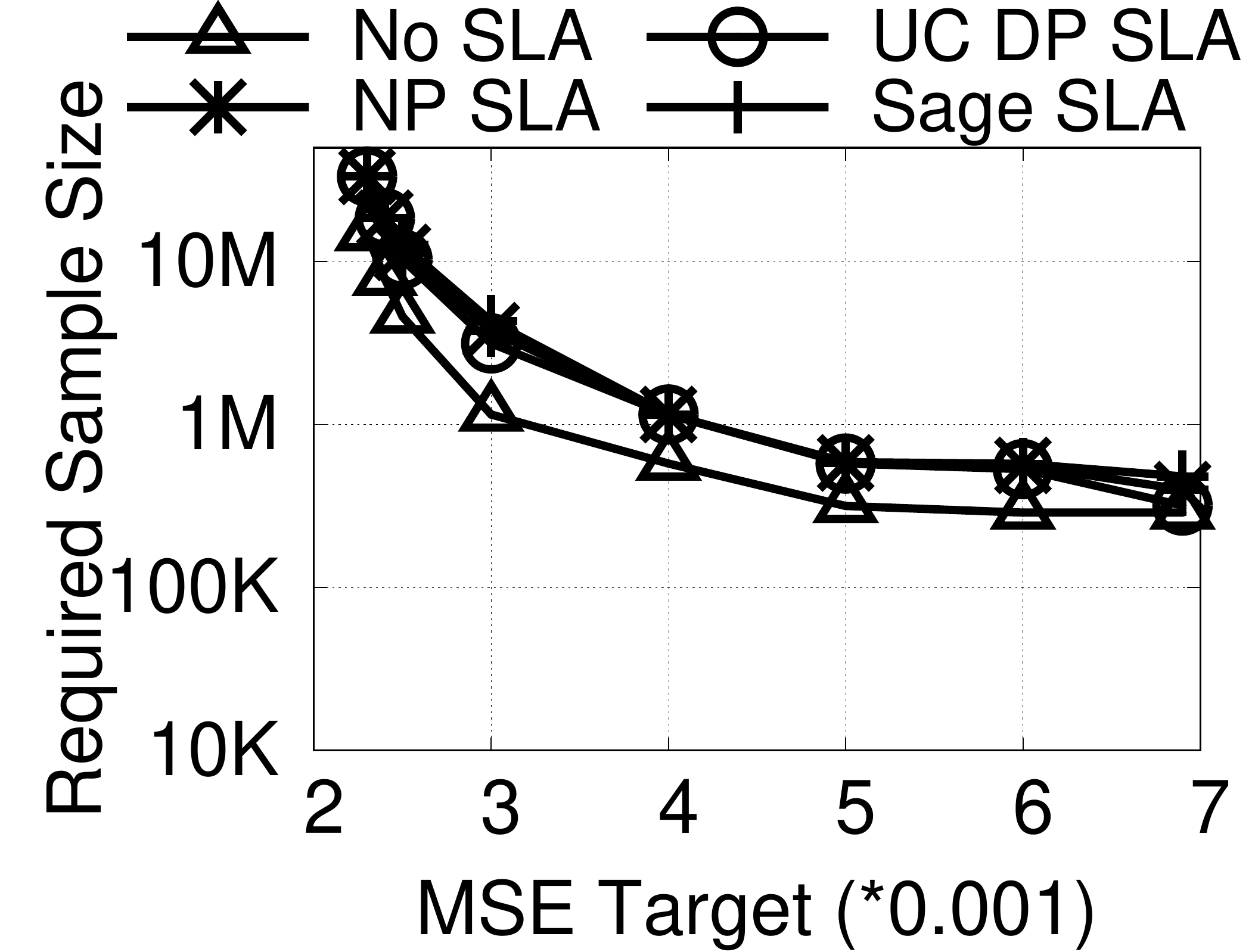}
		\caption{{\bf Taxi NN \ACCEPT}}
		\label{f:dnn_required_sample_complexity}
	\end{subfigure}
\begin{subfigure}[t]{0.245\linewidth}
        \includegraphics[width=\linewidth]{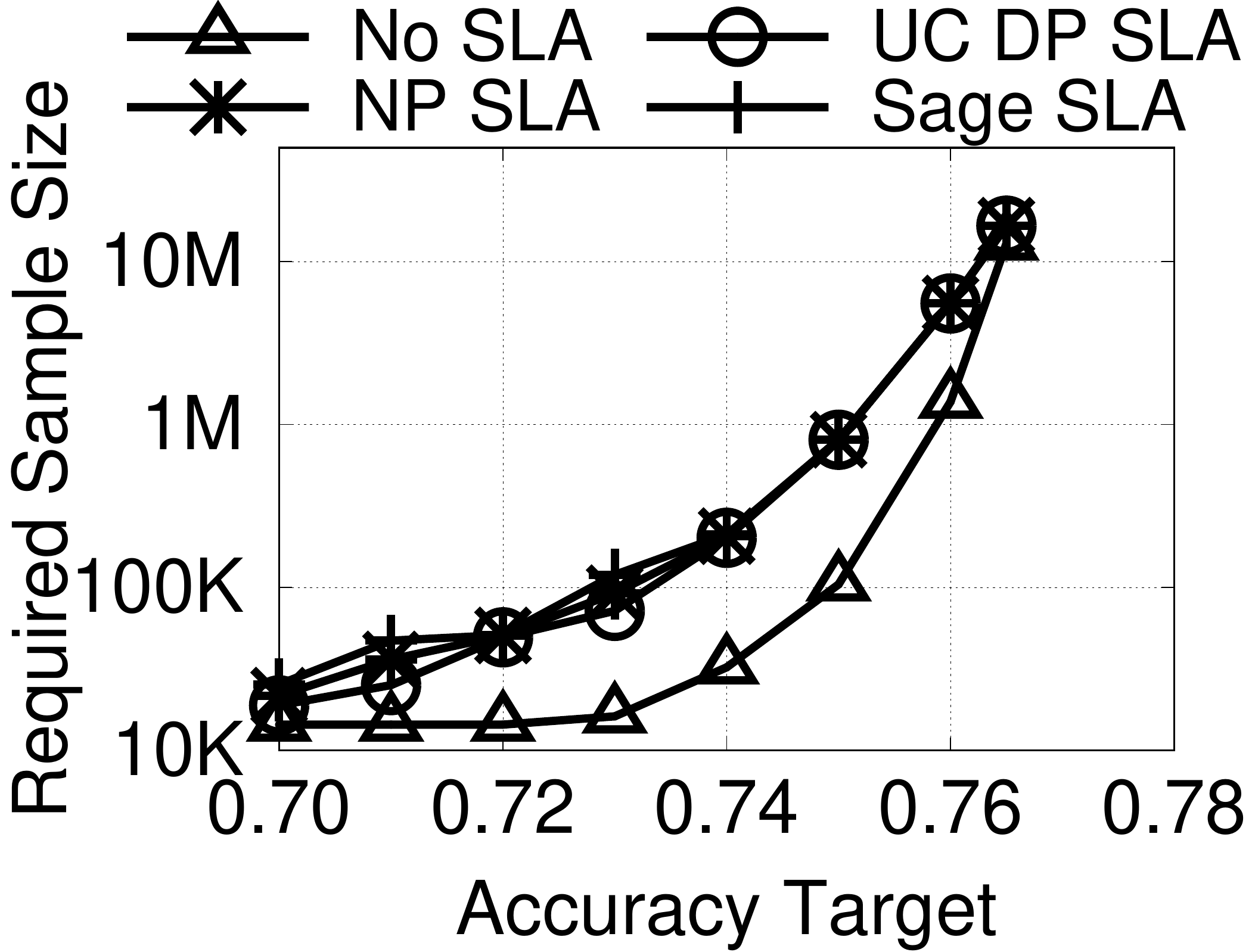}
		\caption{{\bf Criteo LG \ACCEPT}}
		\label{f:criteo_lr_required_sample_complexity_classification}
	\end{subfigure}
  \begin{subfigure}[t]{0.245\linewidth}
        \includegraphics[width=\linewidth]{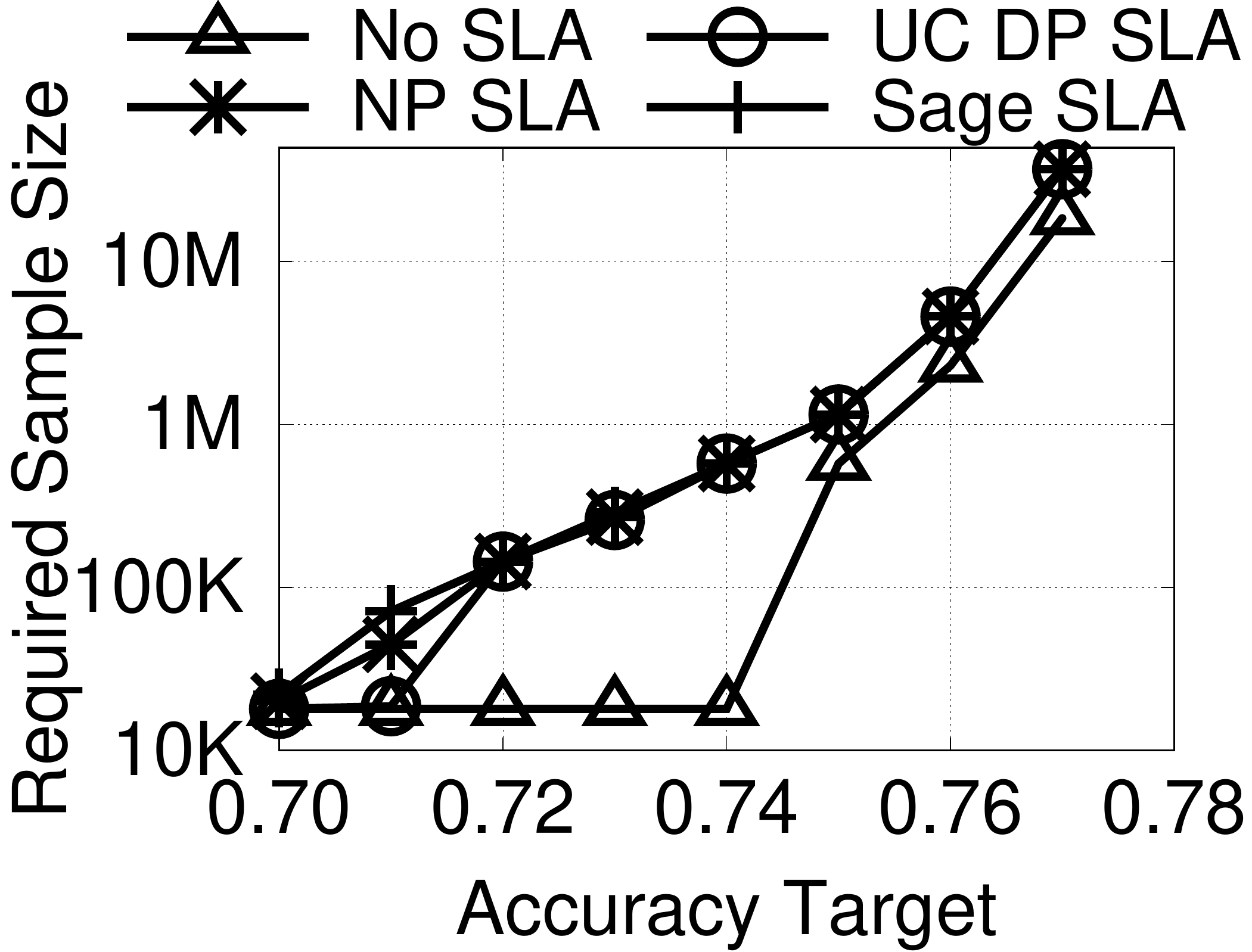}
		\caption{{\bf Criteo NN \ACCEPT}}
		\label{f:criteo_dnn_required_sample_complexity_classification}
	\end{subfigure}
	\vspace{-0.4cm}
	\caption{\footnotesize {\bf Number of Samples Required to \ACCEPT} models at achievable quality targets.
        For MSE targets (Taxi regression \ref{f:lr_required_sample_complexity}, and \ref{f:dnn_required_sample_complexity}) small targets are harder to achieve and require more samples.
        For accuracy targets (Criteo classification \ref{f:criteo_lr_required_sample_complexity_classification}, and \ref{f:criteo_dnn_required_sample_complexity_classification}) high targets are harder and require more samples.
}
	\label{fig:evaluation:data-needed-for-slaed-validation}
	\vspace{-0.3cm}
\end{figure*}

\subsection{Reliability of DP Training Pipelines in \sysname (Q2)}
\label{sec:evaluation:slaed-validation}

\sysname's \iterativetraining and SLAed validation are designed to add reliability to DP model training and validation. However, they may come at a cost of increased data requirements over a non-DP test.
We evaluate reliability and sample complexity for the SLAed validation {\ACCEPT} test.

\T~\ref{tab:evaluation:acceptance-violation-rate} shows the fraction of {\ACCEPT}ed models that violate their quality targets when re-evaluated on the 100K-sample held-out set.
For two confidences $\eta$, we show:
(1) {\em No SLA}, the vanilla TFX validation with no statistical rigor, but where a model's quality is computed with DP.
(2) {\em NP SLA}, a non-DP but statistically rigorous validation. This is the best we can achieve with statistical confidence.
(3) {\em UC DP SLA}, a DP SLAed validation without the correction for DP impact.
(4) {\em \sysname SLA}, our DP SLAed validator, with correction.
We make three observations.
First, the NP SLA violation rates are much lower than the configured $\eta$ values because we use conservative statistical tests.
Second, \sysname's DP-corrected validation accepts models with violation rates close to the NP SLA. Slightly higher for the loss SLA and slightly lower for the accuracy SLA, but {\em well below the configured error rates}.
Third, removing the correction increases the violation rate by 5x for the loss SLA and 20x for the accuracy SLA, violating the confidence thresholds in both cases, at least for low $\eta$.
These results confirm that Sage's SLAed validation is reliable, and that correction for DP is critical to this reliability.

The increased reliability of SLAed validation comes at a cost: SLAed validation requires more data compared to a non-DP test.  This new data is supplemented by \sysname's \iterativetraining.
\F~\ref{f:lr_required_sample_complexity} and \ref{f:dnn_required_sample_complexity} show the amount of train+test data required to {\ACCEPT} a model under various loss targets for the Taxi regression task.
\F~\ref{f:criteo_lr_required_sample_complexity_classification} and \ref{f:criteo_dnn_required_sample_complexity_classification} show the same for accuracy targets for the Criteo classification task.
We make three observations.
First, unsurprisingly, non-rigorous validation (No SLA) requires the least data but has a high failure rate because it erroneously accepts models on small sample sizes.
Second, the best model accepted by Sage's SLA validation are close to the best model accepted by No SLA.
We observe the largest difference in Taxi LR where No SLA accepts MSE targets of $0.0025$ while the Sage SLA accepts as low as $0.0027$.
The best achievable model is slightly impacted by DP, although more data is required.
Third, adding a statistical guarantee but no privacy to the validation (NP SLA) already substantially increases sample complexity.
Adding DP to the statistical guarantee and applying the DP correction incurs limited additional overhead.
The distinction between Sage and NP SLA is barely visible for all but the Taxi LR.
For Taxi LR, adding DP accounts for half of the increase over No SLA requiring twice as much data (one data growth step in \iterativetraining).
Thus, \iterativetraining increases reliability of DP training pipelines for reasonable increase in sample complexity.

\subsection{Benefit of Block Composition (Q3)}
\label{sec:evaluation:block-composition}

\begin{figure*}[t]
	\begin{subfigure}{0.245\linewidth}
		\includegraphics[width=\linewidth]{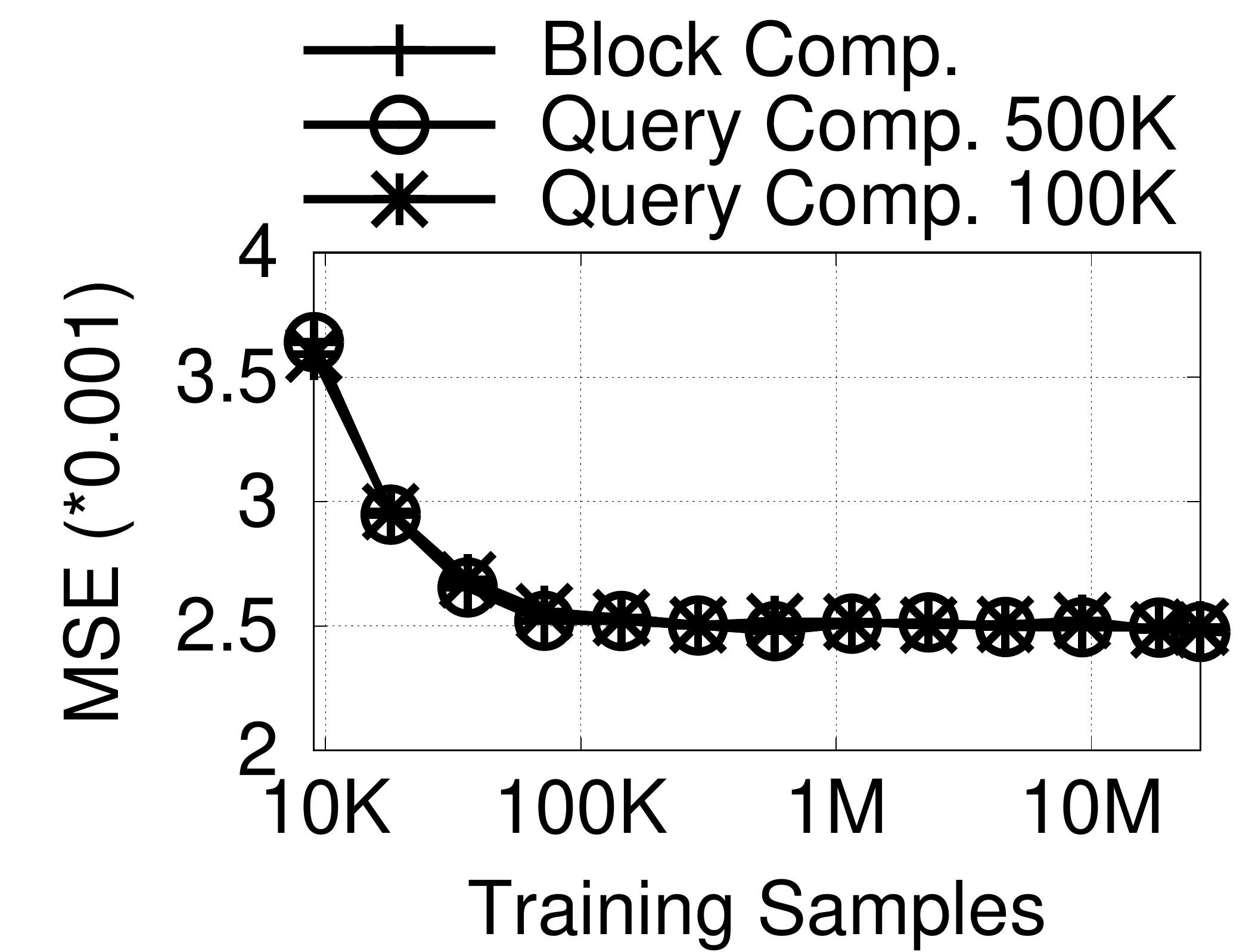}
		\caption{{\bf Taxi LR MSE}}
		\label{fig:evaluation:block-level-accounting-quality-lr}
	\end{subfigure}
	\begin{subfigure}{0.245\linewidth}
        \includegraphics[width=\linewidth]{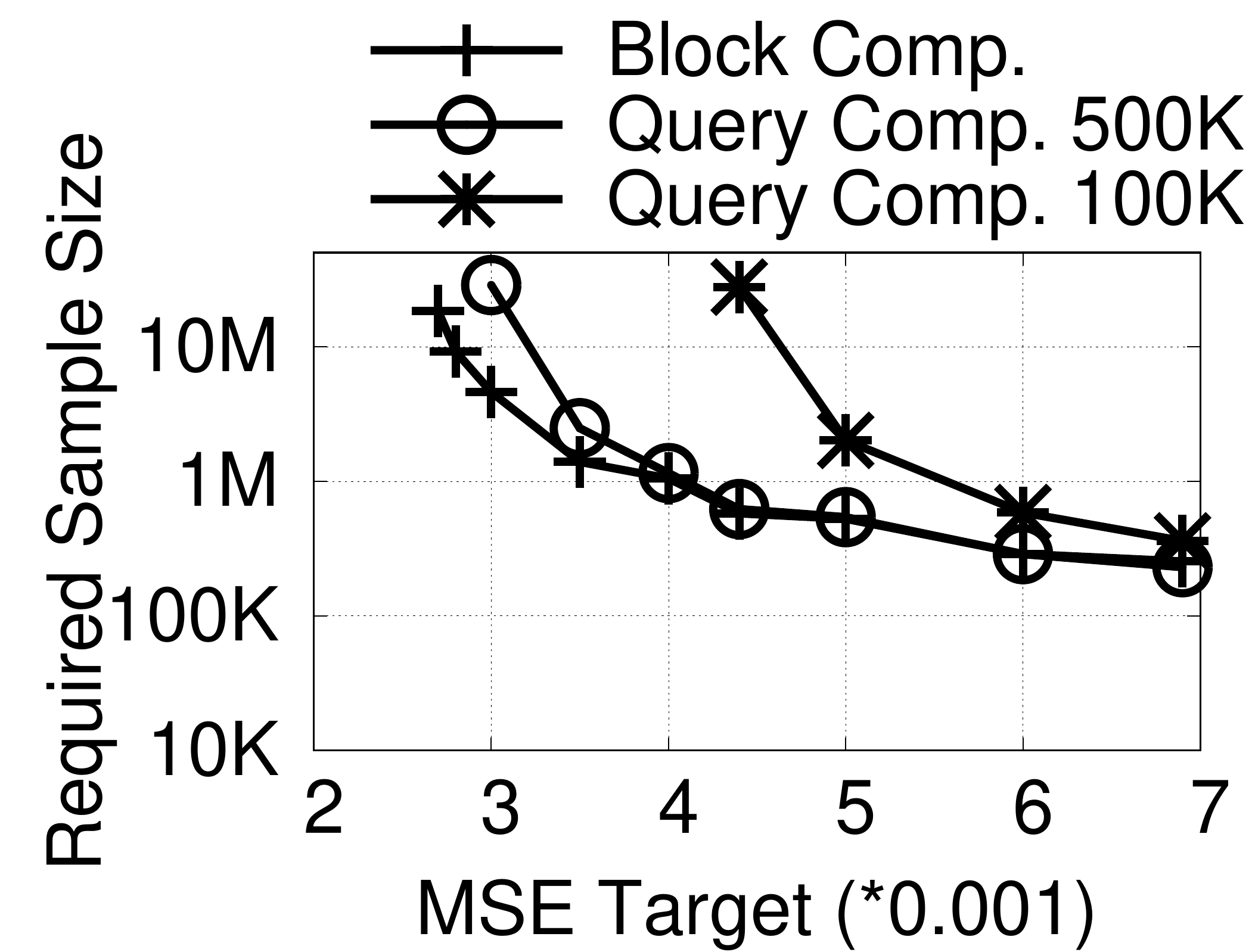}
		\caption{{\bf Taxi LR \ACCEPT}}
		\label{fig:evaluation:block-level-accounting-sc-lr}
	\end{subfigure}
    \begin{subfigure}{0.245\linewidth}
		\includegraphics[width=\linewidth]{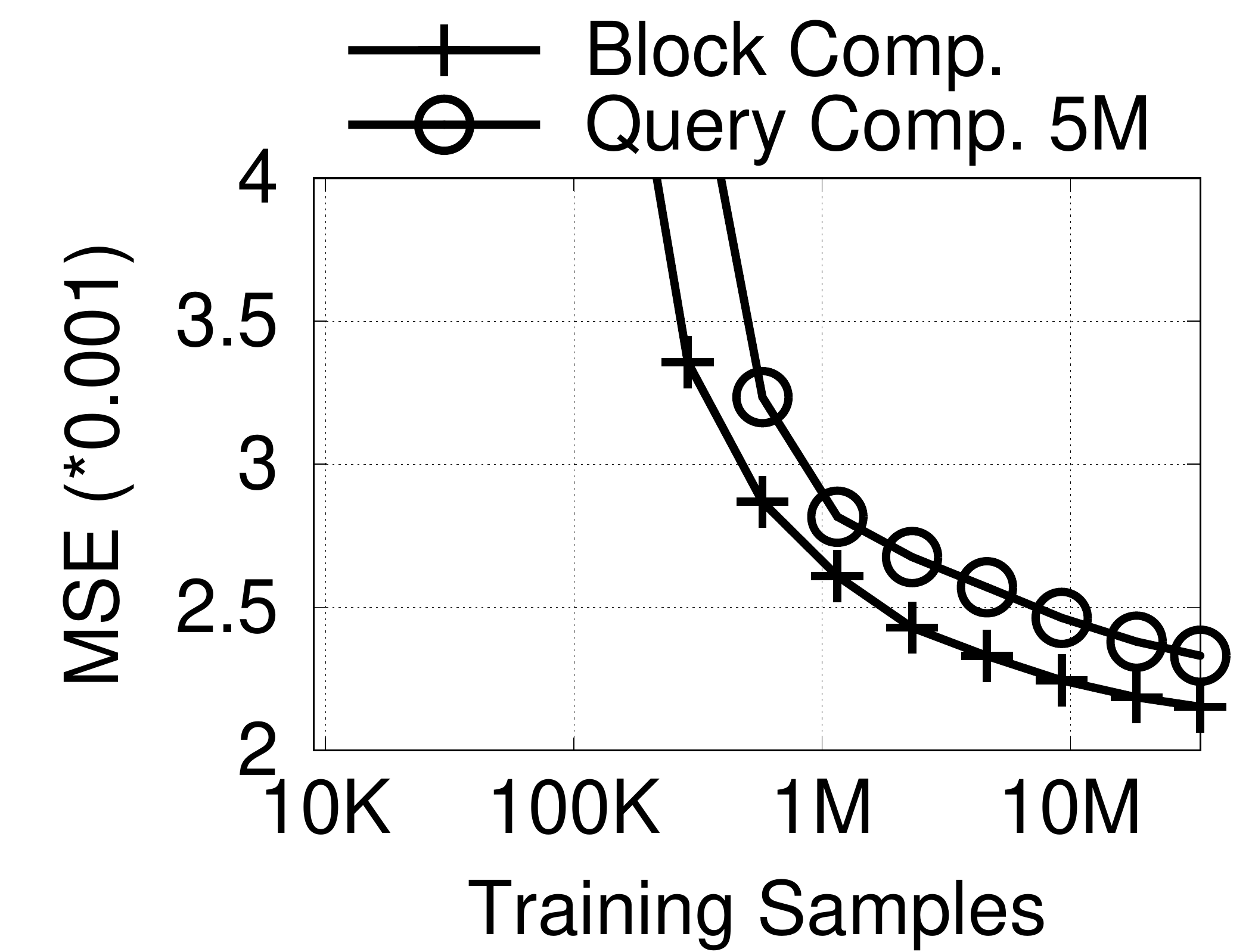}
		\caption{{\bf Taxi DNN MSE}}
		\label{fig:evaluation:block-level-accounting-quality-dnn}
	\end{subfigure}
	\begin{subfigure}{0.245\linewidth}
        \includegraphics[width=\linewidth]{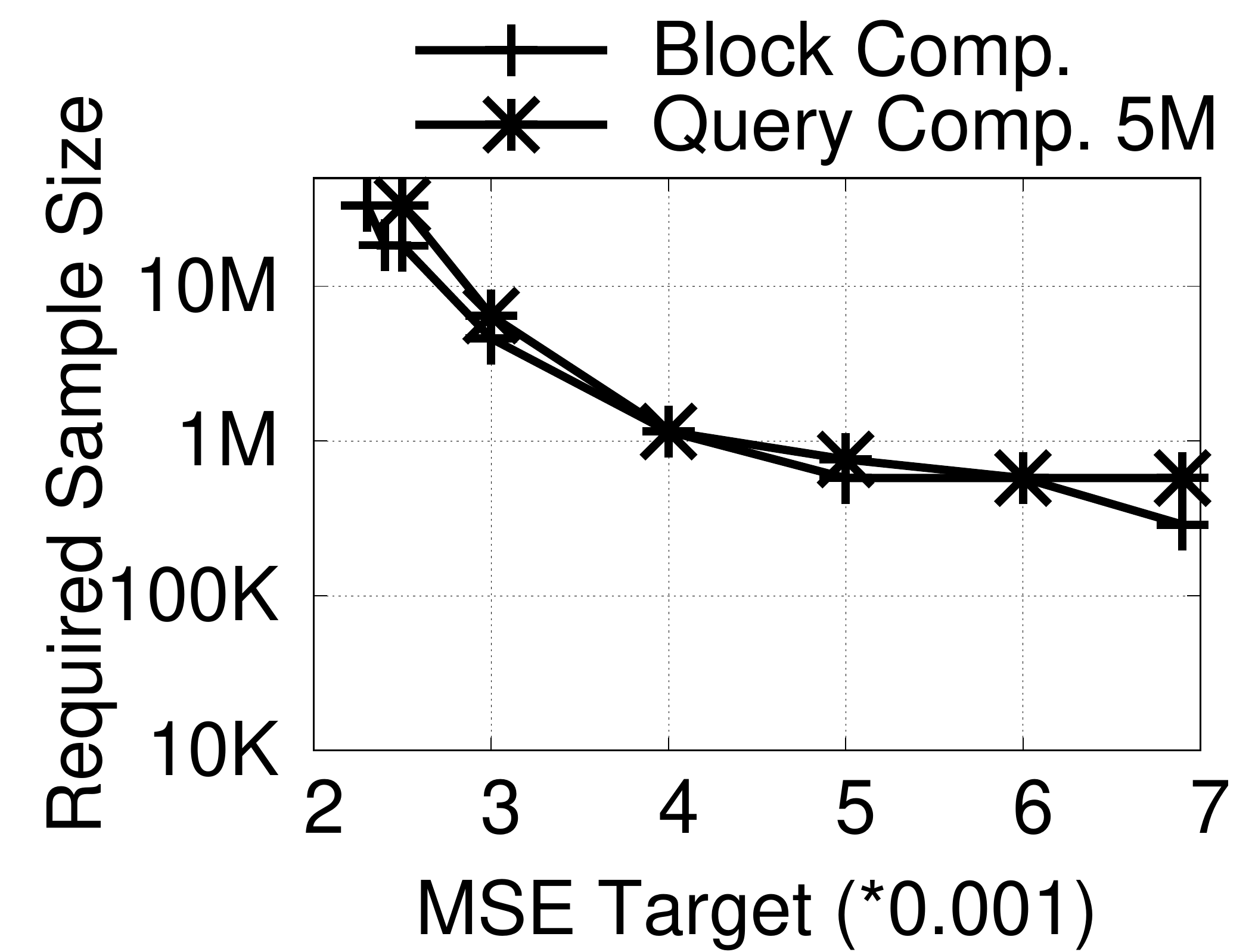}
		\caption{{\bf Taxi DNN \ACCEPT}}
		\label{fig:evaluation:block-level-accounting-sc-dnn}
	\end{subfigure}

	\vspace{-0.4cm}
	\caption{\footnotesize {\bf Block-level vs. Query-level Accounting.} Block-level query accounting provides benefits to model quality and validation.
	}
	\label{fig:evaluation:block-level-accounting}
	\vspace{-0.3cm}
\end{figure*}

Block composition lets us combine multiple blocks into a dataset, such that each DP query runs over all used blocks with only one noise draw.
Without block composition a DP query is split into multiple queries, each operating on a single block, and receiving independent noise. The combined results are more noisy.
\F~\ref{fig:evaluation:block-level-accounting-quality-lr} and~\ref{fig:evaluation:block-level-accounting-quality-dnn} show the model quality of the LR and NN models on the Taxi dataset, when operating on blocks of different sizes, 100K and 500K for the LR, and 5M for the NN.
\F~\ref{fig:evaluation:block-level-accounting-sc-lr} and~\ref{fig:evaluation:block-level-accounting-sc-dnn} show the SLAed validation sample complexity of the same models.
We compare these configurations against \sysname's combined-block training that allows ML training and validation to operate on their full relevance windows.
We can see that block composition helps both the training and validation stages.
While LR training (\F~\ref{fig:evaluation:block-level-accounting-quality-lr}) performs nearly identically for \sysname and block sizes of 100K or 500K (6h of data to a bit more than a day), validation is significantly impacted.
The LR cannot be validated with any MSE better than $0.003$ with block sizes of 500K, and $0.0044$ for blocks of size 100K.
Additionally, those targets that can be validated require significantly more data without \sysname's block composition: 10x for blocks of size 500K, and almost 100x for blocks of 100K.
The NN is more affected at training time. With blocks smaller than 1M points, it cannot even be trained.  Even with an enormous block size of 5M, more than ten days of data (\F~\ref{fig:evaluation:block-level-accounting-quality-dnn}), the partitioned model performs 8\% worse than with \sysname's block composition. Although on such large blocks validation itself is not much affected, the worse performance means that models can be validated up to an MSE target of $0.0025$ (against \sysname's $0.0023$), and requires twice as much data as with block composition.

\subsection{Multi-pipeline Workload Performance (Q4)}
\label{sec:evaluation:end-to-end-experiment}

Last is an end-to-end evaluation of \sysname with a workload consisting of a data stream and ML pipelines arriving over discrete time steps.
At each step, a number of new data points corresponding approximately to 1 hour of data arrives (16K for Taxi, 267K for Criteo).
The time between new pipelines is drawn from a Gamma distribution. When a new pipeline arrives, its sample complexity (number of data points required to reach the target) is drawn from a power law distribution, and a pipeline with the relevant sample complexity is chosen uniformly among our configurations and targets (\T~\ref{t:private-models}).
Under this workload, we compare the average model release in steady state for four different strategies. This first two leverage {\em Query Composition} and {\em Streaming Composition} from prior work, as explained in methodology and \S~\ref{sec:sage-iterative-training}. The other two take advantage of \sysname's Block Composition.
Both strategies uniformly divide the privacy budget of new blocks among all incomplete pipelines, but differ in how each pipeline uses its budget.
{\em Block/Aggressive} uses as much privacy budget as is available when a pipeline is invoked.
{\em Block/Conserve (\sysname)} uses the \IterativeTraining strategy defined in \S~\ref{sec:sage-iterative-training}.

\begin{figure}[t]
	\begin{subfigure}{0.495\linewidth}
    \includegraphics[width=\linewidth]{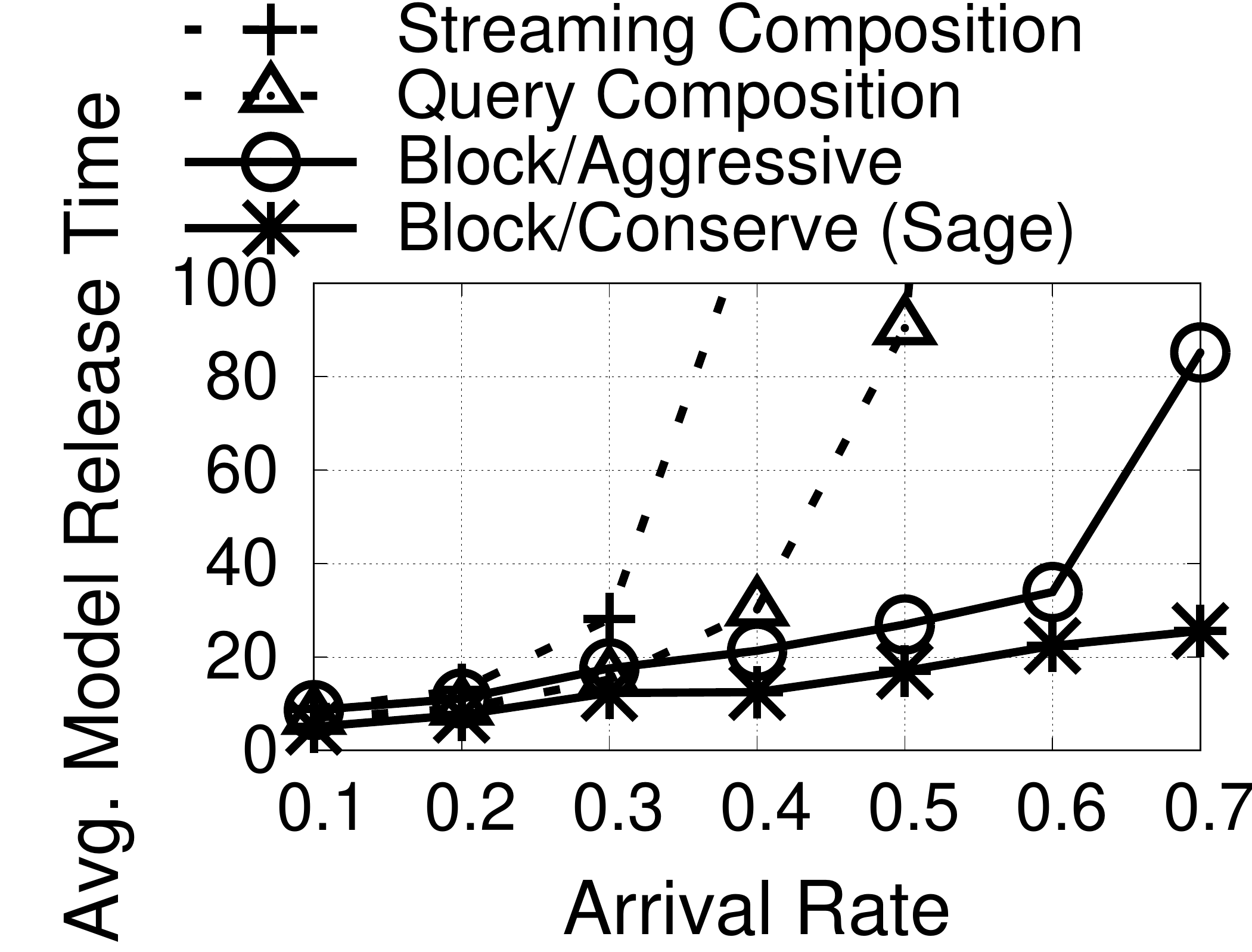}
		\caption{{\bf Taxi Dataset}}
		\label{fig:evaluation:workload-taxi}
	\end{subfigure}
	\begin{subfigure}{0.495\linewidth}
    \includegraphics[width=\linewidth]{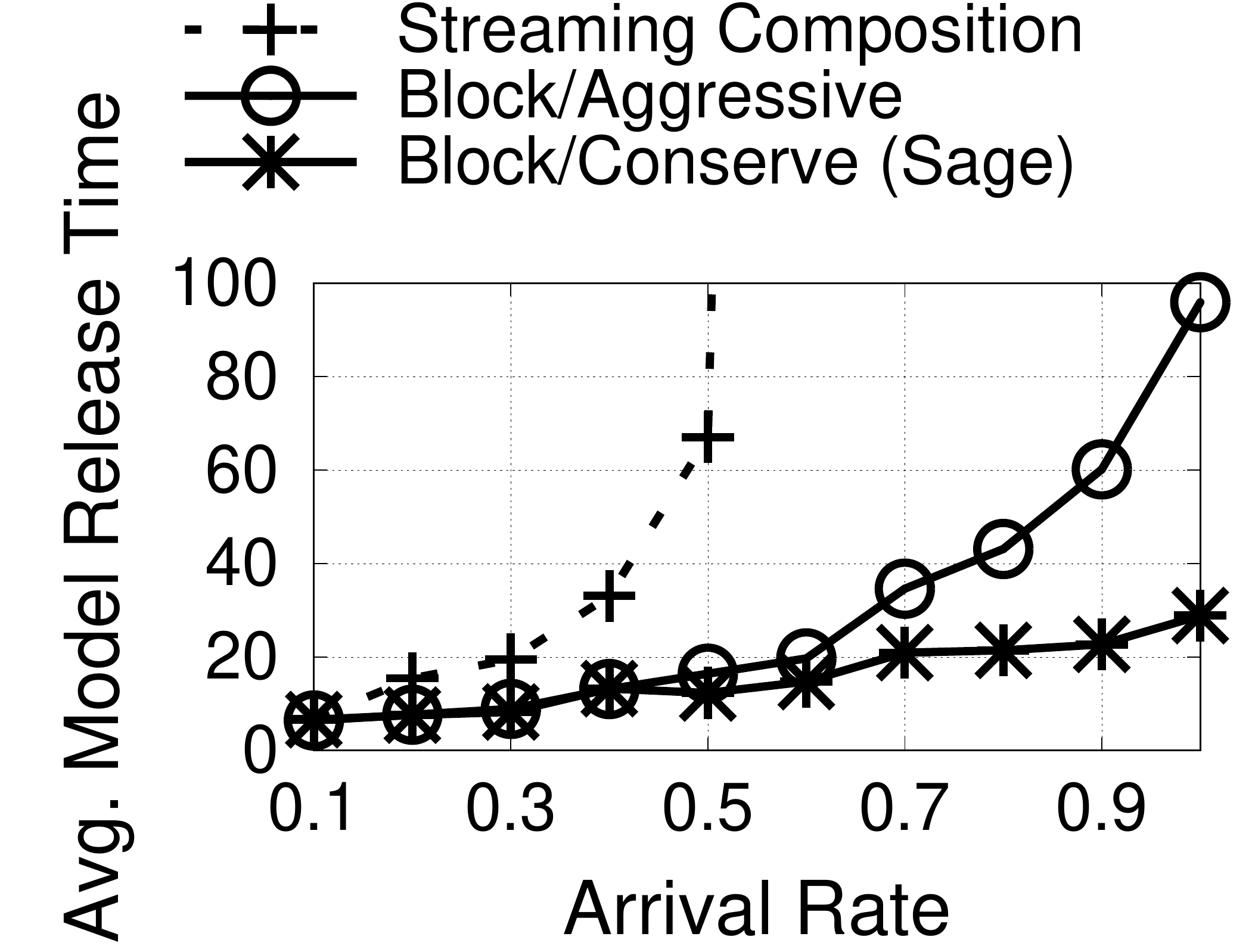}
		\caption{{\bf Criteo Dataset}}
		\label{fig:evaluation:workload-criteo}
	\end{subfigure}

	\vspace{-0.4cm}
  \caption{\bf Average Model Release Time Under Load.}
  \label{fig:evaluation:workload-evaluation}
	\vspace{-0.3cm}
\end{figure}

\F~\ref{fig:evaluation:workload-evaluation} shows each strategy's average model release time under increasing load (higher model arrival rate), as the system enforces $\egdg=(1.0,10^{-6})$-DP over the entire stream.
We make two observations.
First, \sysname's block composition is crucial.
Query Composition and Streaming Composition quickly degrade to off-the-charts release times: supporting more than one model every two hours is not possible and yields release times above three days.
On the other hand, strategies leveraging \sysname's block composition both provide lower release times, and can support up to 0.7 model arrivals per hour (more than 15 new models per day) and release them within a day.
Second, we observe consistently lower release times under the privacy budget conserving strategy.
At higher rates, such as 0.7 new models per hour, the difference starts to grow: Block/Conserve has a release time 4x and 2x smaller than Block/Aggressive on Taxi (\F~\ref{fig:evaluation:workload-taxi}) and Criteo (\F~\ref{fig:evaluation:workload-criteo}) respectively.
Privacy budget conservation reduces the amount of budget consumed by an individual pipeline, thus allowing new pipelines to use the remaining budget when they arrive.

\section{Related Work}
\label{sec:related-work}

\sysname's main contribution -- block composition -- is related to {\em DP composition theory}.
Basic~\cite{Dwork:2006:CNS:2180286.2180305} and strong~\cite{dwork2010boosting,kairouz2015composition} composition theorems give the DP guarantee for multiple queries with adaptively chosen computation.
McSherry~\cite{Mcsherry:pinq} and Zhang, et.al.~\cite{Zhang:2018:EFD:3183713.3196921} show that non-adaptive queries over non-overlapping subsets of data can share the DP budget.
Rogers, et.al.~\cite{rogers2016privacy} analyze composition under adaptive DP parameters, which is crucial to our block composition.
These works all consider fixed datasets and query-level accounting.

Compared to all these works, our main contribution is to formalize the new block-level DP interaction model, which supports ML workloads on growing databases while enforcing a global DP semantic without running out of budget.
This model sits between traditional DP interaction with static data, and streaming DP working only on current data.
In proving our interaction model DP we leverage prior theoretical results and analysis methods.
However, the most comprehensive prior interaction model~\cite{rogers2016privacy} did not support all our requirements, such as interactions with adaptively chosen data subsets, or future data being impacted by previous queries.

{\em Streaming DP}~\cite{Chan:2011:PCR:2043621.2043626,Dwork:2010:DPN:1873601.1873617,pan-private-streaming-algorithms,dwork2010differential} extends DP to data streams but is restrictive for ML.
Data is consumed once and assumed to never be used again.
This enables stronger guarantees, as data need not even be kept internally.
However, training ML models often requires multiple passes over the data.

Cummings, et.al.~\cite{cummings2018differential} consider {\em DP over growing databases}.
They focus on theoretical analysis and study two setups.
In the fist setup, they also run DP workloads on exponentially growing data sizes.
However, their approach only supports linear queries, with a runtime exponential in the data dimension and hence impractical.
In the second setup, they focus on training a single convex ML model and show that it can use new data to keep improving.
Supporting ML workloads would require splitting the privacy budget for the whole stream among models, creating a running out of privacy budget challenge.

A few {\em DP systems} exist, but none focuses on streams or ML.
PINQ~\cite{Mcsherry:pinq} and its generalization wPINQ~\cite{proserpio2014calibrating} give a SQL-like interface to perform DP queries. 
They introduce the partition operator allowing {\em parallel composition}, which resembles \sysname's block composition.
However, this operator only supports non-adaptive parallel computations on non-overlapping partitions, which is insufficient for ML.
Airavat~\cite{roy2010airavat} provides a MapReduce interface and supports a strong threat model against actively malicious developers.
They adopt a perspective similar to ours, integrating DP with access control.
GUPT~\cite{mohan2012gupt} supports automatic privacy budget allocation and lets programmers specify accuracy targets for arbitrary DP programs with a real-valued output; it is hence applicable to computing summary statistics but not to training ML models.
All these works focus on static datasets and adopt a generic, query-level accounting approach that applies to any workload.
Query-level accounting would force them to run out of privacy budget if unused data were available.
Block-level accounting avoids this limitation but applies to workloads with specific data interaction characteristics (\S\ref{sec:sage-access-control}).

\section{Summary and Future Work}
\label{sec:conclusions}
As companies disseminate ML models trained over sensitive data to untrusted domains, it is crucial to start controlling data leakage through these models.
We presented {\em \sysname}, the first ML platform that enforces a global DP guarantee across all models released from sensitive data streams.
Its main contributions are its {\em block-level accounting} that permits endless operation on streams and its {\em \iterativetraining} that lets developers control DP model quality.
The key enabler of both techniques is our systems focus on ML training workloads rather than DP ML's typical focus on individual training algorithms.
While individual algorithms see either a static dataset or an online training regime, workloads interact with {\em growing databases}.  Across executions of multiple algorithms, new data becomes available (helping to renew privacy budgets and allow endless operation) and old data is reused (allowing training of models on increasingly large datasets to lessen the effect of DP noise on model quality).

We believe that this systems perspective on DP ML presents other opportunities worth pursuing in the future.
One of them is to allocate the multiple resources in a DP ML system: data, privacy budgets, and compute resources.
\sysname proposes a specific heuristic for allocating the first two resources (\S\ref{sec:sage-iterative-training}) but leaves unexplored tradeoffs between data and compute resources.
To conserve budgets, we use as much data as is available in the database when a model is invoked, with the lowest privacy budget.  This gives us the best utilization of the privacy resource.  But training on more data consumes more compute resources.
Identifying principled approaches to perform these allocations is an open problem that systems researchers are uniquely positioned to address given the rich resource allocation literature developed by this community.

A key limitation of this work is the focus on event-level privacy, a semantic that is insufficient when groups of correlated observations can reveal sensitive information.
The best known example of such correlation happens when a user contributes multiple observations, but other examples include repeated measurements of a phenomenon over time, or users and their friends on a social network.
In such cases, observations are all correlated and can reveal sensitive information, such as a user's demographic attributes, despite event-level DP.
To increase protection, an exciting area of future work is to add support for and evaluate user-level privacy.
Our block accounting theory is amenable to this semantic (\S\ref{sec:applications-to-user-level-privacy}), but finding settings where the semantic can be sustained without running out of budget is an open challenge.

\section{Acknowledgements}
We thank our shepherd, Thomas Ristenpart, and the anonymous reviewers for the valuable comments.
This work was funded through NSF CNS-1351089, CNS-1514437, and CCF-1740833, two Sloan Faculty Fellowships, a Microsoft Faculty Fellowship, a Google Ph.D. Fellowship, and funds from Schmidt Futures and Columbia Data Science Institute.

\renewcommand{\baselinestretch}{1}
{
  \bibliographystyle{abbrv}
  \bibliography{bib/abbrev,bib/conferences,bib/refs}
}

\appendix
\section{Block Composition}
\label{appendix:block-composition}

This section makes several clarifications and precisions to the block composition theory in \S\ref{sec:block-composition}.
It then ports prior strong composition results to our block accounting model, both for fixed and adaptive choices of blocks and DP parameters.

\subsection{Clarifications and Precisions}

\heading{Neighboring Datasets.}
We measure the distance between datasets using the symmetric difference: viewing
a dataset as a multiset, two datasets  $D$ and $D'$ are neighboring if their
disjunctive union (the elements which are in one of the sets but not in their
intersection) is at most one. We note $|D \oplus D'|$ the distance between $D$
and $D'$.
This definition is not the most standard: most DP work uses the Hamming distance
which counts the number of records to change to go from $D$ to $D'$.
Intuitively, under the symmetric difference an attacker can add or remove a
record in the dataset. Changing a record corresponds to a symmetric difference
of size $2$, but a Hamming distance of $1$.  Changing a record can still be
supported under the symmetric difference using group
privacy \cite{dwork2014algorithmic}: it is thus slightly weaker but as general.

We chose to use the symmetric difference following PINQ \cite{Mcsherry:pinq}, as
this definition is better suited to the analysis of DP composition over splits
of the dataset (our blocks). Changing one record can indeed affect two blocks
(e.g. the timestamp is changed) while adding or removing records only affects
one.

\heading{Neighboring Streams.}
This notion of neighboring dataset extends pretty directly to streams of data~\cite{Chan:2011:PCR:2043621.2043626,Dwork:2010:DPN:1873601.1873617,pan-private-streaming-algorithms,dwork2010differential}.
Two streams $D$ and $D'$ indexed by $t$ are neighboring if there exists an index $T$ such that: for $t < T$ the streams are identical (i.e. $|D_{t<T} \oplus D'_{t<T}| = 0$), and for all $t \geq T$ the streams up to $t$ form neighboring datasets (i.e. $|D_{t \geq T} \oplus D'_{t \geq T}| \leq 1$).
This is equivalent to our Algorithm (\ref{fig:block-compose}) where the data, though unknown, is fixed in advance with only one record being changed between the two streams.

This definition however is restrictive, because a change in a stream's data will typically affect future data. This is especially true in an ML context, where a record changed in the stream will change the targeting or recommendation algorithms that are trained on this data, which in turn will impact the data collected in the future.
Because of this, $D$ and $D'$ will probably differ in a large number of observations following the adversary's change of one record.
Interactions described in Algorithm (\ref{fig:sage-block-compose}) model these dependencies.
We show in Theorem (\ref{theorem:adaptive-block-basic-compose}) that if the data change impacts future data only through DP results (e.g. the targeting and recommendation models are DP) and mechanisms outside of the adversary's control (our ``world'' variable $W$), composition results are not affected.

\heading{Privacy Loss Semantics.}
Recall that bounding the privacy loss (Definition \ref{def:privacy-loss}) with high probability implies DP \cite{kasiviswanathan2014semantics}: if with probability $(1-\delta)$ over draws from $v \sim V^0$ (or $v \sim V^1$) $| \Loss(v) | \leq \epsilon$, then the interaction generating $V^b$ is $(\epsilon, \delta)$-DP.

In this paper, we implicitly treated DP and bounded loss as equivalent by declaring $\M_i$s as $(\epsilon, \delta)$-DP, but proving composition results using a bound on $\M_i$'s privacy loss.
However, this is not exactly true in general, as $(\epsilon, \delta)$-DP implies a bound on privacy loss with weaker parameters, namely that with probability at least $(1-\frac{2\delta}{\epsilon e^\epsilon})$ the loss is bounded by $2\epsilon$.
In practice, this small difference in not crucial, as the typical Laplace and Gaussian DP mechanisms (and those we use in \sysname) do have the same $(\epsilon, \delta)$-DP parameters for their bounds on privacy loss~\cite{dwork2014algorithmic}: the Laplace mechanism for  $(\epsilon, 0)$-DP implies that the privacy loss is bounded by $\epsilon$ and achieving $(\epsilon, \delta)$-DP with the Gaussian mechanism implies that the privacy loss is bounded by $\epsilon$ with probability at least $(1-\delta)$.

\subsection{Proof for Basic Composition}
\label{a:basic-composition}

The paper omits the proof for basic composition to save space.
Here, we re-state the result and spell out the proof:

\begin{theorem*}[Theorem~\ref{theorem:adaptive-block-basic-compose}: Basic Sage Block Composition]
\label{appendix-theorem:adaptive-block-basic-compose}
  AdaptiveStreamBlockCompose($\A$,$b$,$r$,$\epsilon_g$,$\delta_g$,$\World$) is $(\epsilon_g,\delta_g)$-DP if for all $k$, $\text{AccessControl}^k_{\epsilon_g,\delta_g}$ enforces:
{
\begin{equation*}
  \Big( \sum_{\substack{i=1 \\ k \in \block_i}}^r \epsilon_i(v_{<i}) \Big) \le \epsilon_g \ \text{and} \ \Big( \sum_{\substack{i=1 \\ k \in \block_i}}^r \delta_i(v_{<i}) \Big) \le \delta_g .
\end{equation*}
}
\end{theorem*}
\begin{proof}
{Denote $l_i$ the highest block index that existed when
  query $i$ was run. Denote $D^b_{\leq l_i}$ the data blocks that existed at that time.  Recall $v_{<i}$ denotes the results from all queries released previous to $i$.
  Query $i$ depends on both $v_{<i}$ and $D^b_{\leq l_i}$; the latter is a random variable that is fully determined by $v_{<i}$.  Hence, the privacy loss for Alg.~(\ref{fig:sage-block-compose}) is:
  {\footnotesize $\Loss(v) = \ln\Big( \prod\limits_{i=1}^r \frac{P(\O^0_i = v_i | v_{<i}, D^0_{\leq
  			l_i})}{P(\O^1_i = v_i | v_{<i}, D^1_{\leq l_i})} \Big)$ =
  	$\ln\Big( \prod\limits_{i=1}^r
  	\frac{P(\O^0_i = v_i | v_{<i})}{P(\O^1_i = v_i | v_{<i})} \Big)$, 
  }  

  After applying Theorem~\ref{theorem:block-reduction}, we must show that $\forall k,$\\
  $ \big| \ln\Big( \prod\limits_{\substack{i=1 \\ k \in \block_i}}^r \frac{P(\O^0_i = v_i | v_{<i})}{P(\O^1_i = v_i | v_{<i})} \Big) \big| \leq \sum\limits_{\substack{i=1 \\ k \in \block_i}}^r \epsilon_i$ with probability $\ge (1-\sum\limits_{\substack{i=1 \\ k \in \block_i}}^r \delta_i)$.  The justification follows:
{\footnotesize \begin{align*}
  \big| \ln\Big( \prod\limits_{\substack{i=1 \\ k \in \block_i}}^r \frac{P(\O^0_i = v_i | v_{<i})}{P(\O^1_i = v_i | v_{<i})} \Big) \big| \leq \sum\limits_{\substack{i=1 \\ k \in \block_i}}^r \big| \ln\Big( \frac{P(\O^0_i = v_i | v_{<i})}{P(\O^1_i = v_i |v_{<i})} \Big) \big|
\end{align*}
}  Since $\M_i$ is $\epsilon_i(v_{<i}), \delta_i(v_{<i}))$-DP, $\big| \ln\Big( \frac{P(\O^0_i = v_i | v_{<i})}{P(\O^1_i = v_i |v_{<i})} \Big) \big| \leq \epsilon_i(v_{<i})$ with probability $\ge 1-\delta_i(v_{<i})$. Applying a union bound over all queries for which $k \in \block_i$ concludes the proof.
}  \end{proof}

\subsection{Strong Composition with Block-Level Accounting}
\label{a:strong-composition}

We now prove strong composition results for Algorithms (\ref{fig:block-compose}) and (\ref{fig:sage-block-compose}).

\heading{Fixed Blocks and DP Parameters:}
We now show how to use advanced composition results (e.g. \cite{dwork2010boosting}) is the context of block composition.
This approach requires that both the blocks used by each query and the DP parameters be fixed in advanced, and correspond to Algorithms (\ref{fig:block-compose}).

\begin{theorem}[Strong Block Composition -- Fixed DP Parameters]
  \label{corollary:block-strong-compose}
    BlockCompose($\A$, $b$, $r$, $(\epsilon_i,\delta_i)_{i=1}^r$, $\text{blocks}_{i=1}^r$) is $(\epsilon_g, \delta_g)$-DP, with:
    \begin{align*}
  \epsilon_g & = \max_k \Big( \sum_{\substack{i=1 \\ k \in blocks_i}}^r (e^{\epsilon_i} - 1)\epsilon_i + \sqrt{\sum\limits_{\substack{i=1 \\ k \in blocks_i}}^r 2\epsilon_i^2\log(\frac{1}{\tilde\delta})} \Big) , \\
  \delta_g & = \tilde\delta + \max_k \Big( \sum_{\substack{i=1 \\ k \in blocks_i}}^r \delta_i \Big)
    \end{align*}
\end{theorem}
\begin{proof}
{
  After applying Theorem~\ref{theorem:block-reduction}, what remains to be shown is that
  $\forall k, \ \big| \ln\Big( \prod\limits_{\substack{i=1 \\ i \in \block_k}}^r \frac{P(\O^0_i = v_i | v_{<i})}{P(\O^1_i = v_i | v_{<i})} \Big) \big|
  \leq
\sum_{\substack{i=1 \\ k \in blocks_i}}^r (e^{\epsilon_i} - 1)\epsilon_i + \sqrt{\sum\limits_{\substack{i=1 \\ k \in blocks_i}}^r 2\epsilon_i^2\log(\frac{1}{\tilde\delta})}$
, with probability at least $(1- \tilde\delta \sum\limits_{\substack{i=1 \\ k \in \block_i}}^r \delta_i)$.
  Using the fact $\M_i$'s privacy loss is bounded by $\epsilon_i$ with probability at least $(1-\delta_i)$, we know that there exists events $E_i$ and $E'_i$ with joint probability at least $(1-\delta_i)$ such that
  for all $v_i$, $\Big| \ln\big( \frac{P(\O^0_i=v_i|E_i)}{P(\O^1_i=v_i|E'_i)} \big) \Big| \leq e^\epsilon$.
  We can now condition the analysis on $E_i$ and $E'_i$, and use Theorem 3.20 of~\cite{dwork2014algorithmic} to get that with probability at least $(1-\tilde\delta)$, $\Big| \ln\big( \frac{P(\O^0=v|E_i)}{P(\O^1=v|E'_i)} \big) \Big| \leq e^{\epsilon_k}$, where $\epsilon_k = \sum_{\substack{i=1 \\ k \in blocks_i}}^r (e^{\epsilon_i} - 1)\epsilon_i + \sqrt{\sum\limits_{\substack{i=1 \\ k \in blocks_i}}^r 2\epsilon_i^2\log(\frac{1}{\tilde\delta})}$.
  A union bound on the $E_i$ and $E'_i$ for all $\{i, k \in blocks_i \}$ completes the proof.
}  \end{proof}

The proof directly extends to the stream setting (yellow parts of Alg.~(\ref{fig:block-compose}) in the same way as in the proof of Theorem~\ref{theorem:adaptive-block-basic-compose}.

\heading{Adaptive Blocks and DP Parameters:}
Recall that with either adaptive blocks or DP parameters (or both), DP
parameters $(\epsilon_i(v_{<i}), \delta_i(v_{<i}))$ depend on history.
Traditional composition theorems do not apply to this setting. For basic
composition, the DP parameters still ``sum'' under composition. However, as
showed in \cite{rogers2016privacy}, strong composition yields a different formula:
while the privacy loss still scales as the square root of the number of queries,
the constant is worse than with parameters fixed in advance.

\begin{theorem}[Strong Adaptive Stream Block Composition]
\label{corollary:adaptive-block-strong-compose}
  AdaptiveStreamBlockCompose($\A$, $b$, $r$, $\epsilon_g$, $\delta_g$, $\W$) is $(\epsilon_g, \delta_g)$-DP, and:
{
\begin{align*}
  \hspace{-1.1cm}
  \max_k \Big( \sum_{\substack{i=1 \\ k \in \block_i}}^r \frac{(e^{\epsilon_i} - 1)\epsilon_i}{2} + \sqrt{2\big( \sum_{\substack{i=1 \\ k \in \block_i}}^r \epsilon_i^2 + \frac{\epsilon_g^2}{28.04 \log(1/\tilde\delta)} \big)} \\
  \sqrt{\big( 1 + \frac{1}{2} \log(\frac{28.04 \log(1 / \tilde\delta) \sum\limits_{\substack{i=1 \\ k \in \block_i}}^r \epsilon_i^2}{\epsilon_g^2} + 1) \log(\frac{1}{\tilde\delta}) \big)} \Big) \leq \epsilon_g, \\
  \tilde\delta + \max_k \Big( \sum_{\substack{i=1 \\ k \in \block_i}}^r \delta_i \Big) \leq \delta_g
\end{align*}
}
\end{theorem}
\begin{proof}
{
  Similarly to the proof of Theorem \ref{theorem:adaptive-block-basic-compose}, we apply Theorem \ref{theorem:block-reduction}, and bound the privacy loss of any block $k$ using Theorem 5.1 of \cite{rogers2016privacy}.
}  \end{proof}

\section{Validation Tests}
\label{appendix:sla-tests}

\sysname has built-in validators for three classes of metrics. \S\ref{sec:sage-iterative-training} describes the high level approach and the specific functionality and properties of the loss-based validator.
This section details all three validators and proves their statistical and DP guarantees.

\subsection{SLAed Validator for Loss Metrics}
\label{appendix:sla-tests:loss}

Denote a loss function $l(f, x, y)$ with range $[0,B]$ measuring the quality of a prediction $f(x)$ with label $y$ (lower is better), and a target loss $\tau_{loss}$ on the data distribution $\D$.

\heading{{\ACCEPT} Test:}
Given a the DP-trained model $\fdp$, we want to release $\fdp$ only if $\cL_{\D}(\fdp) \triangleq \Exp_{(x, y) \sim \D} l(\fdp, x, y) \leq \tau_{loss}$.
The test works as follows. First, compute a DP estimate of the number of samples in the test set, corrected for the impact of DP noise to be a lower bound on the true value with probability $(1-\frac{\eta}{3})$ (Lines 11-13 \L\ref{list:dp_validator}):
\[
  \ntestDPLB = \ntest + \Lap(\frac{2}{\epsilon}) - \frac{2}{\epsilon}\ln(\frac{3}{2\eta}) .
\]
Then, compute a DP estimate of the loss corrected for DP impact (Lines 14-17 \L\ref{list:dp_validator}) to be an upper bound on the true value, $\cL_{te}(\fdp) \triangleq \frac{1}{\ntest} \sum_{te} l(\fdp, x, y)$, with probability $(1-\frac{\eta}{3})$:
\[
  \LUB^{dp}_{te}(\fdp) = \frac{1}{\ntestDPLB} \Big( \sum_{te} l(\fdp, x, y) + \Lap(\frac{2B}{\epsilon}) + \frac{2B}{\epsilon}\ln(\frac{3}{2\eta}) \Big)
\]

Lines 18-20, we see that \sysname will {\ACCEPT}~when:
\[
\LUB^{dp}_{te}(\fdp) + \sqrt{\frac{2 B \LUB^{dp}_{te}(\fdp) \ln(3/\eta)}{\ntestDPLB}} + \frac{4 B \ln(3/\eta)}{\ntestDPLB} \leq \tau_{loss} .
\]
This test gives the following guarantee:

\begin{proposition}[Loss {\ACCEPT} Test (same as Proposition~\ref{prop:loss-accept})]\label{aprop:loss-accept}
  With probability at least $(1-\eta)$, the {\ACCEPT} test returns true only if $\cL_{\D}(\fdp) \leq \tau_{loss}$.
\end{proposition}
\begin{proof}[Proof]
  The corrections for DP noise imply that $P(\ntestDPLB > \ntest) \leq \frac{\eta}{3}$, and $P(\cL^{dp}_{te}(\fdp) > \cL_{te}(\fdp)) \leq \frac{\eta}{3}$ (i.e. the lower bounds hold with probability at least (1-$\frac{\eta}{3}$)).
  Define $\text{UB}^{dp} \triangleq \LUB^{dp}_{te}(\fdp) + \sqrt{\frac{2 B \LUB^{dp}_{te}(\fdp) \ln(3/\eta)}{\ntestDPLB}} + \frac{4 \ln(3/\eta)}{\ntestDPLB}$,
  and $\text{UB} \triangleq \cL_{te}(\fdp) + \sqrt{\frac{2 B \cL_{te}(\fdp) \ln(3/\eta)}{\ntest}} + \frac{4 \ln(3/\eta)}{\ntest}$.
  Applying Bernstein's inequality \cite{Shalev-Shwartz:2014:UML:2621980}
  yields $P(\cL_{\D}(\fdp) > \text{UB}) \leq \frac{\eta}{3}$.
  A union bound on those three inequalities gives that with probability at least
  $(1-\eta)$, $\cL_{\D}(\fdp) \leq \text{UB} \leq \text{UB}^{dp}$. The test {\ACCEPT}s
  when $\text{UB}^{dp} \leq \tau_{loss}$.  \end{proof}

This test uses Bernstein's concentration inequality to compute an upper bound for the loss over the entire distribution~\cite{Shalev-Shwartz:2014:UML:2621980}, which will give good bounds if the loss is small.
If instead one expects the variance to be small, one can use empirical Bernstein bounds~\cite{2009arXiv0907.3740M} as a drop-in replacement.
Otherwise, one can fall back to Hoeffding's inequality~\cite{hoeffding1963probability}.

\heading{{\REJECT} Test:}
The \REJECT test terminates \IterativeTraining of a model when no model of the considered class $\cF$ can hope to achieve the desired $\tau_{loss}$ performance.
Noting the best possible model on the data distribution $\fopt \triangleq \arg\min_{f\in\cF} \cL_{\D}(f)$, we want to reject when $\cL_{\D}(\fopt) > \tau_{loss}$.
To do this, we consider the best model in $\cF$ on the training set $\hat f \triangleq \arg\min_{f\in\cF} \cL_{tr}(f)$, and proceed as follows. First, we compute the DP upper and lower bounds for $\ntrain$, holding together with probability at least $(1-\frac{\eta}{3})$ (Lines 24-27 \L\ref{list:dp_validator}):
\begin{align*}
  \ntrainDP & = \ntrain + \Lap(\frac{2}{\epsilon}) , \\
  \ntrainDPLB & = \ntrainDP - \frac{2}{\epsilon}\ln(\frac{3}{\eta}) , \\
  \ntrainDPUB & = \ntrainDP + \frac{2}{\epsilon}\ln(\frac{3}{\eta}) .
\end{align*}
Then, we compute a DP estimate of the loss corrected for DP impact (Lines 14-17 \L\ref{list:dp_validator}) to be a lower bound on the true value with probability $(1-\frac{\eta}{3})$:
\[
  \LLB^{dp}_{te}(\hat f) = \frac{1}{\ntrainDPUB} \Big( \sum_{tr} l(\hat f, x, y) + \Lap(\frac{2B}{\epsilon}) - \frac{2B}{\epsilon}\ln(\frac{3}{2\eta}) \Big) .
\]
Finally, \sysname will \REJECT when:
\[
  \LLB^{dp}_{te}(\hat f) - B \sqrt{\frac{\log(3/\eta)}{\ntrainDPLB}} > \tau_{loss} .
\]

This test gives the following guarantee:
\begin{proposition}[Loss {\REJECT} Test]\label{aprop:loss-reject}
  With probability at least $(1-\eta)$, $\fdp$ (or more accuratly $\cF$) is rejected only if $\cL_{\D}(\fopt) > \tau_{loss}$.
\end{proposition}
\begin{proof}[Proof]
  By definition, $\cL_{tr}(\hat f) \leq \cL_{tr}(\fopt)$.
  Applying Hoeffding's inequality~\cite{hoeffding1963probability} to $\fopt$ gives $P(\cL_{\D}(\fopt) < \cL_{tr}(\fopt) -  B \sqrt{\frac{\log(3/\eta)}{\ntrain}}) \leq \frac{\eta}{3}$.
  Similarly to the proof for Proposition (\ref{aprop:loss-accept}), applying a union bounds over this inequality and the DP correction gives that with probability at least $(1-\eta)$, $\cL_{\D}(\fopt) \geq \cL_{tr}(\fopt) - B \sqrt{\frac{\log(3/\eta)}{\ntrain}} \geq \cL^{dp}_{tr}(\hat f) - B \sqrt{\frac{\log(3/\eta)}{\ntrainDPLB}} > \tau_{loss}$, concluding the proof.
\end{proof}

Here we leverage Hoeffding's inequality~\cite{hoeffding1963probability} as we need to bound $\cL_{\D}(\fopt)$: since we do not know $\fopt$, we cannot compute its variance or loss on the training set to use (empirical) Bernstein bounds.

\heading{DP Guarantee:}
We now need to prove that the {\ACCEPT} and {\REJECT} tests each are $(\epsilon, 0)$-DP. Since they each use a disjoint split of the dataset (training and testing) running both tests is $(\epsilon, 0)$-DP over the entire dataset.

\begin{proposition}[DP \ACCEPT]\label{aprop:accept-dp}
  {\ACCEPT} is $(\epsilon, 0)$-DP.
\end{proposition}
\begin{proof}[Proof]
  The only data interaction for {\ACCEPT} are to compute $\ntestDPLB$ and $\LUB^{dp}_{te}(\fdp)$.
  The former is sensitivity $1$, and is made $\frac{\epsilon}{2}$-DP with the Laplace mechanism.
  The latter uses data in the inner sum, which has sensitivity $B$ and is made $\frac{\epsilon}{2}$-DP with the Laplace mechanism.
  Using basic composition, the test is $(\epsilon, 0)$-DP.
\end{proof}

The proof for \REJECT is a bit more complicated as part of the procedure involves computing $\hat f$, which is not DP.
\begin{proposition}[DP \REJECT]\label{aprop:reject-dp}
  {\REJECT} is $(\epsilon, 0)$-DP.
\end{proposition}
\begin{proof}[Proof]
  To apply the same argument as for Prop.\ref{aprop:accept-dp},
  we need to show that computing $\cL_{tr}(\hat f)$, which includes computing $\hat f$, has sensitivity $B$.
  Recall that $\hat f$ minimizes the training loss:
  $\hat f_{tr} = \arg\min_{f\in\cF} \sum_{\ntrain} l(f, x, y)$.
  If we add a data point $d$, then $\hat f_{tr}$ has a loss at worst $B$ on the new point.
  Because the best model with the new data point $\hat f_{tr+d}$ is at least as good as $\hat f_{tr}$ (otherwise it wouldn't be the $\arg\min$ model), $\cL_{tr + d}(\hat f_{tr+d}) \leq \cL_{tr}(\hat f_{tr}) + B$.
  Similarly, $\hat f_{tr+d}$ cannot be better than $\hat f_{tr}$ on the training set, so $\cL_{tr + d}(\hat f_{tr+d}) \geq \cL_{tr}(\hat f_{tr})$.
  Hence the sensitivity of computing $\cL_{tr}(\hat f)$ is at most $B$.
\end{proof}

We highlight that \REJECT relies on computing $\hat f$ for both its statistical and DP guarantees. However, $\hat f$ may not always be available. In particular, it can be computed for convex problems, but not for NNs for instance.

\subsection{SLAed Validator for Accuracy}
\label{appendix:sla-tests:accuracy}

The accuracy metric applies to classification problems, where a model predicts one out of many classes.  The target $\tau_{acc}$ is the proportion of correct predictions to reach on the data distribution.
The accuracy validator follows the same logic as the loss validator, with two small changes. First, the upper and lower bounds are reversed as losses are minimized while accuracy is maximized. Second, the result of classification predictions follows a binomial distribution, which yields tighter confidence intervals.
Note $\mathbb{1}\{\text{predicate}\}$ the indicator function with value $1$ if the predicate is true, and $0$ otherwise, and $\BinUB(k, n, \eta)$ and $\BinLB(k, n, \eta)$ the upper and lower bounds on the probability parameter $p$ of a binomial distribution such that $k$ happens with probability at least $\eta$ out of $n$ independent draws (both can be conservatively approximated using a Clopper-Pearson interval).

We compute $k^{dp}_{te} = \sum_{\ntest} \mathbb{1}\{f(x) = y\} + Laplace(\frac{2}{\epsilon})$ and $\ntestDP = \ntest + Laplace(\frac{2}{\epsilon})$. Similarly for the training set $tr$, $k^{dp}_{tr} = \sum_{\ntest} \mathbb{1}\{\hat f(x) = y\} + Laplace(\frac{2}{\epsilon})$ and $\ntrainDP = \ntrain + Laplace(\frac{2}{\epsilon})$. \sysname will {\ACCEPT} when:
\[
  \BinLB\Big(k^{dp}_{te} - \frac{2}{\epsilon}\ln(\frac{3}{\eta}), \ntestDP + \frac{2}{\epsilon}\ln(\frac{3}{\eta}), \frac{\eta}{3}\Big) \geq \tau_{acc} .
\]
And {\REJECT} when:
\[
  \BinUB\Big(k^{dp}_{tr} + \frac{2}{\epsilon}\ln(\frac{3}{\eta}), \ntrainDP - \frac{2}{\epsilon}\ln(\frac{3}{\eta}), \frac{\eta}{3}\Big) < \tau_{acc} .
\]

Using the same reasoning as before, we can prove the equivalent four Propositions (\ref{aprop:loss-accept}, \ref{aprop:loss-reject}, \ref{aprop:accept-dp}, \ref{aprop:reject-dp}) for the accuracy validator.
However, finding $\hat f$ for accuracy is computationally hard.

\subsection{SLAed Validator for Sum-based Statistics}
\label{appendix:sum-validation}
This validator applies to computing sum based statistics (e.g. mean, variance).
The target is defined as the maximum size of the absolute (additive) error
$\tau_{err}$ for these summary statistics on the data distribution. This error
can be computed directly on the training set, so there is no need for a test set.
Second, because of the law of large numbers, we can always reach the target
precision, so there is no rejection test.
We next show the test using Hoeffding's inequality
\cite{hoeffding1963probability} (if we expect the variance to be small,
empirical Bernstein bounds \cite{2009arXiv0907.3740M} will be better and are
directly applicable).

Compute:
\vspace{-0.3cm}
\[
  \ntrainDP = \ntrain + Laplace(\frac{2}{\epsilon}) - \frac{2}{\epsilon}\ln(\frac{2}{\eta})
\]
\vspace{-0.3cm}

\sysname {\ACCEPT}s if:
\vspace{-0.3cm}
\[
  \frac{1}{\ntrainDP}\frac{2}{\epsilon}\ln(\frac{2}{\eta})
   + B \sqrt{\frac{\ln(2/\eta)}{\ntrainDP}} \leq \tau_{err} ,
\]
in which case the absolute error is bellow $\tau_{err}$ with probability at
least $(1-\eta)$, accounting for the statistical error, as well as the impact of
DP noise on both the sum based summary statistic and the {\ACCEPT} test.
Once again, the same reasoning allows us to prove equivalent Propositions to \ref{aprop:loss-accept} and \ref{aprop:accept-dp}.

\section{Data Cleaning}
The NYC taxi data has a lot of outliers. We filtered the dataset based on the
following criteria: Prices outside of the range $[\$0, \$1000]$, durations
outside of $[0, 2.5]$ hours, malformed date strings, Points falling outside of
the box bounded by ($40.923$, $-74.27$) in the northwest and ($40.4$, $-73.65$)
in the southeast which encompasses all of New York City and much of the
surrounding metropolitan area.  Filtering points is acceptable within
differential privacy but requires that we also account for privacy loss on the
filtered points -- which we do.
However, \sysname does not address the data exploration phase which would be
required for more complex data filtering.
Exploration is outside of \sysname's scope, and may be better addressed by
systems such as PINQ~\cite{Mcsherry:pinq}.

\end{document}